%% file: main.tex
\documentclass{article}
\usepackage{iclr2027_conference,times}
\iclrfinalcopy

\input{math_commands.tex}

\usepackage{hyperref}
\usepackage{url}
\usepackage{booktabs}
\usepackage{colortbl}
\usepackage{graphicx}
\graphicspath{{figs/}}
\usepackage{amsmath}
\usepackage{amsthm}
\usepackage{thmtools,thm-restate}
\usepackage{multirow}
\usepackage{enumitem}
\usepackage{caption}
\usepackage{etoc}

\newtheorem{proposition}{Proposition}
\newtheorem{lemma}{Lemma}
\usepackage[most]{tcolorbox}
\definecolor{accent}{HTML}{1F5288}
\hypersetup{colorlinks=true,linkcolor=accent,citecolor=accent,urlcolor=accent}
\tcolorboxenvironment{proposition}{
  enhanced,
  colback=accent!4, colframe=accent!50,
  boxrule=0.5pt, arc=2.5pt,
  left=7pt, right=7pt, top=3pt, bottom=3pt, boxsep=0pt,
  before skip=6pt, after skip=6pt}
\tcolorboxenvironment{restatable}{
  enhanced,
  colback=accent!4, colframe=accent!50,
  boxrule=0.5pt, arc=2.5pt,
  left=7pt, right=7pt, top=3pt, bottom=3pt, boxsep=0pt,
  before skip=6pt, after skip=6pt}
\newtcolorbox{restatedbox}{
  enhanced,
  colback=accent!4, colframe=accent!50,
  boxrule=0.5pt, arc=2.5pt,
  left=7pt, right=7pt, top=3pt, bottom=3pt, boxsep=0pt,
  before skip=6pt, after skip=6pt}
\tcolorboxenvironment{lemma}{
  enhanced,
  colback=accent!4, colframe=accent!50,
  boxrule=0.5pt, arc=2.5pt,
  left=7pt, right=7pt, top=3pt, bottom=3pt, boxsep=0pt,
  before skip=6pt, after skip=6pt}

\newif\ifinappendix
\newcommand{\proofin}[1]{\ifinappendix\else#1\fi}

\newtcolorbox{promptbox}{enhanced, colback=black!3, colframe=black!25,
  boxrule=0.4pt, arc=2pt, left=6pt, right=6pt, top=3pt, bottom=3pt,
  fontupper=\small, before skip=5pt, after skip=7pt}

\definecolor{tabsub}{gray}{0.45}
\definecolor{tabrule}{gray}{0.62}
\newcommand{\hmidrule}{\arrayrulecolor{tabrule}\midrule\arrayrulecolor{black}}
\newcommand{\hcmidrule}[1]{\arrayrulecolor{tabrule}#1\arrayrulecolor{black}}
\newcommand{\tsub}[1]{{\color{tabsub}#1}}
\newcommand{\ci}[2]{#1\makebox[2.2em][l]{$^{\textcolor{tabsub}{\pm\text{#2}}}$}}
\newcommand{\oursrow}{\rowcolor{accent!7}}

\newcommand{\apptocline}[5]{{\leftskip=#1\relax \rightskip=25pt plus 1fil \parfillskip=-\rightskip
   \parindent=0pt \hangindent=#2\relax
   \makebox[#2][l]{#3}#4\nobreak\hfill\nobreak\makebox[25pt][r]{#5}\par}}
\newcommand{\appendixcontents}{\begingroup
  \hypersetup{hidelinks}\setlength{\parskip}{0pt}\etocsettagdepth{main}{none}\etocsettagdepth{appendix}{subsection}\etocsettocstyle{\noindent{\large\sc Appendix}\par
    \vspace{1.2ex}{\color{accent!40}\hrule height 0.4pt}\vspace{0.2ex}}{\vspace{2.2ex}{\color{accent!40}\hrule height 0.4pt}\vspace{3.5ex}}\etocsetstyle{section}{}{\vspace{2.4ex}}{\apptocline{0pt}{20pt}{\textcolor{accent}{\etocnumber}}{\sc\etocname}{\textcolor{black!50}{\etocpage}}}{}\etocsetstyle{subsection}{\vspace{0.7ex}}{\vspace{0.6ex}}{{\small\apptocline{20pt}{22pt}{\textcolor{black!50}{\etocnumber}}{\etocname}{\textcolor{black!50}{\etocpage}}}}{}\tableofcontents
  \endgroup}

\newcommand{\tv}{\mathrm{TV}}
\newcommand{\mmd}{\mathrm{MMD}}

\title{Sample What You Say: Aligning Language Models to Sample the Distributions They State}

\author{Kasra Arabi \\ New York University \And Virginia Smith \\ Carnegie Mellon University \And Chhavi Yadav \\ Carnegie Mellon University}

\begin{document}
\etocdepthtag.toc{main}

\maketitle

\begin{abstract}
\input{main_files/kabstract}
\end{abstract}

\section{Introduction}
\label{sec:intro}
\input{main_files/kintro}

\section{Preliminaries}
\label{sec:setup}
\input{main_files/ksetup}

\section{Method}
\label{sec:method}

GRPO needs a reward for each rollout, but whether the model samples correctly depends on the group as a whole (Section~\ref{sec:setup}). Section~\ref{sec:witness} derives such a reward, the witness advantage, from the maximum mean discrepancy. Section~\ref{sec:optimizes} shows that, in expectation, training with this advantage performs gradient descent on the distance to the target distribution. Section~\ref{sec:simpler} shows that the group-scalar reward optimizes a biased objective and that the sign witness does not train outcomes with small target probability toward their targets.

\subsection{The Witness Advantage}
\label{sec:witness}
\input{main_files/kwitness}

\subsection{The Witness Advantage Trains Toward the Target Distribution}
\label{sec:optimizes}
\input{main_files/kwitnessopt}

\subsection{Where Simpler Rewards Fail}
\label{sec:simpler}
\input{main_files/ksimplerrewards}

\section{Experiments}
\label{sec:experiments}

We first describe the targets, the metric, and the baselines (Section~\ref{sec:impl}). We then train on synthetic distribution families, where the target is known exactly and we control which parameters and families the model sees in training (Section~\ref{sec:synthetic}). On these targets, we compare the witness advantage with the baselines on the training targets, then test whether training transfers to unseen parameters and families, and then compare it with the two simpler rewards of Section~\ref{sec:simpler}. Lastly, we test the witness advantage in two settings closer to real applications (Section~\ref{sec:real}).

\subsection{Setup}
\label{sec:impl}
\input{main_files/kimplementationdetails}

\subsection{Results on Synthetic Distributions}
\label{sec:synthetic}
\input{main_files/ksynthfamilies}

\subsection{Results on Tasks Closer to the Real World}
\label{sec:real}
\input{main_files/kopiniondist}

\section{Discussion}
\label{sec:discussion}

\input{main_files/kdiscussion}

\clearpage

\bibliography{iclr2027_conference}
\bibliographystyle{iclr2027_conference}

\clearpage
\appendix
\inappendixtrue
\etocdepthtag.toc{appendix}
\appendixcontents
\clearpage

\input{main_files/kappendix}

\end{document}

%% file: math_commands.tex
\usepackage{amsmath,amsfonts,bm}

\def\eqref#1{equation~\ref{#1}}

\def\1{\bm{1}}

\DeclareMathAlphabet{\mathsfit}{\encodingdefault}{\sfdefault}{m}{sl}
\SetMathAlphabet{\mathsfit}{bold}{\encodingdefault}{\sfdefault}{bx}{n}

\newcommand{\E}{\mathbb{E}}

%% file: main_files/kabstract.tex
Language models are increasingly used to sample from a specified distribution, for instance, to simulate survey respondents or generate synthetic data. Instruction-tuned models can state such a distribution correctly and still fail to sample from it. Prompting and changes to decoding reduce this mismatch only partly, which motivates training with policy optimization. Group relative policy optimization (GRPO) is a natural fit for this problem because it already samples a group of rollouts per prompt, and the group's empirical distribution can be compared with the target. However, scoring the group as a whole gives every rollout the same reward. Group-relative centering then sets all advantages to zero, and the model receives no learning signal. To give each rollout its own signal, we introduce the \textit{witness advantage}, a per-rollout advantage derived from maximum mean discrepancy (MMD). It trains a model to match a target distribution over a finite set of outcomes. The MMD between the model's distribution and the target has a witness function that measures how over- or under-produced each outcome is. Each rollout's advantage estimates the negative witness at its outcome, so a rollout is rewarded for an outcome the group under-produces and penalized for one it over-produces. The witness advantage is computed in closed form from the group's outcome counts, and we use it as the reward in GRPO. On unseen target distributions, training with the witness advantage substantially reduces the total variation distance to the target while largely preserving the model's general capabilities.

%% file: main_files/kintro.tex
Many real-world applications of a language model require it to generate samples from a specified distribution, for instance when simulating survey respondents \citep{santurkar2023whose,meister2025benchmarking}, choosing actions according to a specified policy \citep{misaki2025ssot}, or generating names, numbers, or coin flips \citep{zhang2024forcing,vankoevering2024random}.

However, current models can identify the specified distribution but fail to sample from it. In our experiments, we evaluate an instruction-tuned model, Qwen2.5-1.5B-Instruct \citep{yang2024qwen25}. It correctly identifies the family and parameters of the specified distribution for five of six common distribution families (Figure~\ref{fig:phenomenon}b), yet its samples remain far from that distribution. For a biased coin with $P(\text{Heads})=0.005$, the model places probability $0.678$ on Heads (Figure~\ref{fig:phenomenon}a). The error is therefore in the model's output probabilities, not in the sampling step, so changing the decoding rule can only partly correct it. The best training-free method we test removes only $19\%$ of this error (Section~\ref{sec:synthetic}).
Prior work finds this mismatch across many model families \citep{meister2025benchmarking,gu2026illusion,zhao2026dice,jang2026cannotsample}. It also finds evidence that post-training moves a model's samples further from the specified distribution \citep{gu2026illusion,jang2026cannotsample}. Prompting can improve sampling from stated distributions \citep{misaki2025ssot,zhang2025verbalized}, but it has limited success in our experiments (Section~\ref{sec:synthetic}). To correct the mismatch between the stated and sampled distributions, we must therefore change the model's probabilities themselves.

Policy optimization \citep{ouyang2022training,shao2024deepseekmath,guo2025deepseekr1} is a natural way to change these probabilities. It updates a model using rewards on the model's own samples, so it suits objectives that are easy to evaluate but hard to supervise explicitly. Distribution matching fits this setting, since whether a model's samples match a target distribution can only be evaluated over a group of samples. This motivates our central question:
\vspace{-0.05in}
\begin{center}
\emph{Can policy optimization teach a language model to sample from the distributions it states, and what property of the training signal enables it to do so?}
\end{center}

\begin{figure}[t]
\centering
\includegraphics[width=\linewidth]{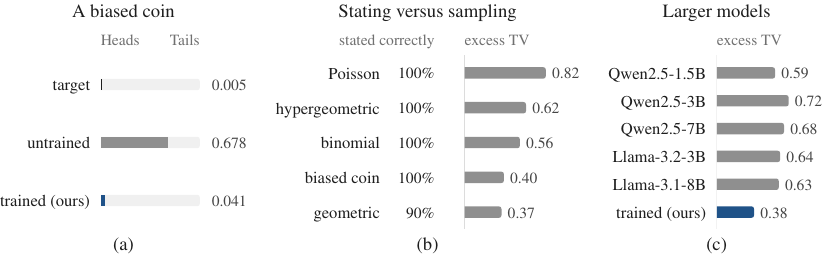}
\caption{Instruction-tuned models state distributions they cannot sample. (a)~For a coin with $P(\text{Heads})=0.005$, Qwen2.5-1.5B-Instruct places next-token probability $0.678$ on Heads, and $0.041$ after training with the witness advantage. (b)~On five families, the untrained model states the family and parameters correctly for $90\%$ to $100\%$ of targets, yet its median excess TV is $0.37$ to $0.82$. Excess TV is the total variation (TV) distance to the stated target minus the expected TV of a perfect sampler with the same number of draws. (c)~Median excess TV on the 100 unseen-family targets. Larger untrained models are no better than the 1.5B model. Training the 1.5B model with the witness advantage lowers it to $0.38$ (Section~\ref{sec:synthetic}).}
\label{fig:phenomenon}
\end{figure}

The central design choice in policy optimization is the reward. To train a model to sample from a target distribution, the reward must measure how well the model's samples match that distribution. This is challenging because a distribution is a property of many draws, whereas standard post-training rewards assign a score to each response individually. A natural approach is therefore to sample a group of outputs, compare their empirical distribution to the target distribution, and assign the resulting score to every sample in the group.

Group relative policy optimization (GRPO) \citep{shao2024deepseekmath} is well suited to this setting because it already samples a group of rollouts for each prompt. However, GRPO centers advantages within each group, so giving every rollout the same distribution-level reward makes every advantage zero (Lemma~\ref{lem:degeneracy}). One plausible workaround is to divide the rollouts into several subgroups, score each subgroup according to its empirical distribution, and center these scores across subgroups. Training with this workaround does not lower the error (Figure~\ref{fig:training}a). Supervised cross-entropy training on the target probability mass function (PMF) can make the model match the target distribution \citep{zhang2024forcing}, but in our experiments it degrades the general capabilities of the model (Section~\ref{sec:synthetic}). We therefore need per-rollout credit assignment that still reflects how well the group matches the target distribution.

To provide such a per-rollout training signal, we use the witness function of maximum mean discrepancy (MMD). MMD measures the discrepancy between two distributions, and its maximizing function, the witness function, identifies at each outcome how much it is over- or under-produced relative to the target \citep{gretton2012kernel}. On a finite outcome set with the exact-match kernel, the witness at an outcome is the difference between its model probability and its target probability. For each rollout, we estimate the negative witness function at its outcome using leave-one-out frequencies within the group, and use this estimate as the rollout's reward in GRPO. A rollout therefore receives a positive reward when its outcome is under-produced relative to the target and a negative reward when it is over-produced. Unlike a group-level reward, this signal assigns credit to individual outcomes while still reflecting the mismatch of the group as a whole. Before GRPO subtracts the group mean, the resulting update is an unbiased estimate of the negative gradient of the squared MMD between the model's outcome distribution and the target, and the subtraction adds only a bias of order $1/G$ in the group size $G$ (Proposition~\ref{prop:witness}). Section~\ref{sec:synthetic} shows that the size of this credit matters for outcomes with small target probability.

To summarize, our key contributions are as follows.
\begin{enumerate}[leftmargin=*,itemsep=0pt,topsep=1pt]

\item We show that prompting and decoding have limited success in making a model sample from the distributions it states, because the model's probabilities over outcomes are already wrong. We therefore use policy optimization to correct them (Section~\ref{sec:synthetic}).

\item We introduce the witness advantage, a per-rollout advantage that enables GRPO to learn this task, and prove that, in expectation, training with it moves the model toward the target distribution. We also show when simpler rewards fail (Section~\ref{sec:method}).

\item We show that one to three GPU-hours of training substantially improve sampling on the training distributions, on unseen parameters and distribution families, and on stated opinion distributions and urn draws, while largely preserving general capabilities (Section~\ref{sec:experiments}).

\end{enumerate}

\textbf{Related Work.} Prior work closest to ours also scores each GRPO rollout using the other rollouts in its group. The reward of GAPO \citep{anschel2025gapo} assumes a uniform target and penalizes a rollout by the fraction of the group that shares its outcome. Because this fraction always counts the scored rollout among those sharing its outcome, it is a biased estimate of the model's probability of that outcome. Distribution-aware reward \citep{park2026dar} is designed for numeric regression and scores a group of predicted numbers against the correct number for each input. Its target is a single number, whereas ours is a probability distribution. \citet{jang2026cannotsample} list rewards scored over groups of samples as future work, but do not define or evaluate such a reward. Appendix~\ref{app:related} discusses further related work.

%% file: main_files/ksetup.tex
\textbf{Notation.}
A prompt describes a target distribution $q$ over a finite set of outcomes $\mathcal{X}$ and asks the model to draw one outcome from it. The prompt specifies the family of $q$, its parameters, and usually its PMF (Appendix~\ref{app:setup}). Given the prompt, the model generates a text response $y$ with probability $P_\theta(y)$. A parser $\varphi$ extracts from each response $y$ the outcome $\varphi(y)$ that it states. The model's \emph{outcome distribution}, $\pi_\theta(x) = \sum_{y:\,\varphi(y) = x} P_\theta(y)$, is the probability that its response states outcome $x$. Our goal is to make $\pi_\theta$ match $q$. Some responses do not state a valid outcome, for example an outcome with extra words around it or a number outside the support of $q$. The parser maps them to the symbol $\bot$ (Appendix~\ref{app:setup}). We add $\bot$ to $\mathcal{X}$ and set $q(\bot)=0$.

\textbf{Error measure.}
We measure how far a distribution $p$ over outcomes is from the target $q$ with the total variation (TV) distance, $\tv(p, q) = \tfrac12 \sum_x |p(x) - q(x)|$. TV is the largest difference between the probabilities that $p$ and $q$ assign to the same set of outcomes, and it ranges from $0$ to $1$. We usually cannot compute the model's outcome distribution exactly, so we estimate its TV from the empirical distribution of $n$ sampled outcomes (Section~\ref{sec:impl}).

\textbf{GRPO.}
We train with GRPO \citep{shao2024deepseekmath}. For each prompt, GRPO samples a group of $G$ rollouts from the model, with responses $y_1,\dots,y_G$. We write $x_i = \varphi(y_i)$ for the outcome of rollout $i$, and $\hat p(x) = \#\{i : x_i = x\}/G$ for the fraction of rollouts whose outcome is $x$. GRPO scores each rollout with a scalar reward $r_i$ and sets its \emph{advantage} to the reward minus the group mean,
\begin{equation}
\tilde A_i = r_i - \mathrm{mean}(r_1,\dots,r_G).
\label{eq:grpoadv}
\end{equation}
Like \citet{liu2025understanding}, we omit the original GRPO's division by the standard deviation of the rewards (Appendix~\ref{app:training}). Training increases the probability of rollouts with a positive advantage and decreases the probability of rollouts with a negative advantage. Ignoring GRPO's PPO-style clipping \citep{schulman2017proximal} and its KL penalty, the update direction is proportional to $\tfrac1G\sum_i \tilde A_i \nabla_\theta \log P_\theta(y_i)$, where $\nabla_\theta \log P_\theta(y_i)$ is the direction in parameter space that most increases the log-probability of response $y_i$, and the update scales this direction by $\tilde A_i$. Since the reward affects the update only through $\tilde A_i$, we can change what the model learns by changing only the reward while keeping the rest of GRPO as is.

\textbf{Per-rollout rewards.}
Equation~\ref{eq:grpoadv} requires a reward $r_i$ for every rollout. In math reasoning, the task GRPO was introduced for, this reward checks whether the rollout's answer is correct. Sampling from a distribution has no correct answer to check. Every outcome with positive target probability is a valid draw. What matters is how often the model produces each outcome. An outcome is over-produced if the model's probability of it is above its target probability, and under-produced if it is below.

%% file: main_files/kwitness.tex
To train a model to sample from a target distribution, we need a function that tells us, for each outcome, whether the model produces it too often or too rarely, and by how much. Given such a function, each rollout $i$ receives an advantage computed from the function's value at its own outcome $x_i$. The advantage is positive if $x_i$ is under-produced and negative if $x_i$ is over-produced. Training then produces more of the under-produced and less of the over-produced outcomes, shifting the outcome distribution toward the target distribution. We also want this function to come from a distance between the model's outcome distribution and the target, so that we can prove training reduces this distance. The maximum mean discrepancy (MMD) of \citet{gretton2012kernel} gives such a distance, together with a function of this kind, called the witness function.

MMD measures the discrepancy between two distributions as the largest difference in the expectation of a test function under them. Let $p$ and $q$ be distributions on an outcome space $\mathcal{X}$, and let $k$ be a kernel with reproducing kernel Hilbert space $\mathcal{H}$. The MMD is
\[
\mmd(p,q)
=
\sup_{\|f\|_{\mathcal{H}} \le 1}
\left(
\E_{x \sim p}[f(x)] - \E_{x \sim q}[f(x)]
\right).
\]
The maximizing function $f^\star$, called the \emph{witness function}, is proportional to the difference between the mean embeddings of $p$ and $q$:
\[
f^\star \propto \mu_p - \mu_q,
\qquad
\mu_p := \E_{x \sim p}[k(x,\cdot)],
\quad
\mu_q := \E_{x \sim q}[k(x,\cdot)].
\]

Because our outcome space is finite, we use the exact-match kernel
$k(x,x')$, which is $1$ if $x=x'$ and $0$ otherwise, and for which the mean embedding is the PMF itself. Let $\pi_\theta(x)$ be the model probability of outcome $x$, and let $q(x)$ be its target probability. Then
\begin{equation}
\mmd^2(\pi_\theta,q)
=
\sum_{x \in \mathcal{X}}
\bigl(\pi_\theta(x) - q(x)\bigr)^2,
\qquad
f^\star(x) \propto \pi_\theta(x) - q(x).
\end{equation}
Thus, $\mmd^2$ is the squared Euclidean distance between the model and target PMFs, and it is zero only when $\pi_\theta = q$. The witness at $x$ is the model's excess probability on that outcome, $\pi_\theta(x) - q(x)$. It is positive when $x$ is over-produced, negative when it is under-produced, and its magnitude is how far the model's probability of $x$ is from its target.

The witness depends on the model's outcome distribution $\pi_\theta$, which we cannot compute exactly. Instead, we estimate it from the group of $G$ rollouts, where $x_i$ is the outcome of rollout $i$. The simplest estimate of $\pi_\theta(x_i)$ is the fraction of rollouts in the group whose outcome is $x_i$. This fraction counts rollout $i$ itself, and counting it adds a term of order $1/G$ to the expected update. The added term increases the probability that two draws from the model have the same outcome, pushing the model to concentrate its probability mass on fewer outcomes (Appendix~\ref{app:fullgroup}). We remove this term by leaving rollout $i$ out. The \emph{leave-one-out frequency} $\hat p_{-i}(x_i)$ is the fraction of the other $G-1$ rollouts whose outcome is $x_i$. Applying the same correction gives an unbiased estimator of $\mmd^2$ \citep{gretton2012kernel,binkowski2018demystifying}.

We then take the negative of the witness, so an under-produced outcome receives a positive advantage and an over-produced outcome receives a negative advantage. Rollout $i$ with outcome $x_i$ receives the \emph{witness advantage}
\begin{equation}
A_i = 2\,\big(q(x_i) - \hat p_{-i}(x_i)\big),
\label{eq:witness}
\end{equation}
where $q(x_i)$ is the target probability of the rollout's outcome, and $\hat p_{-i}(x_i)$ is how often the other rollouts produced that outcome. A rollout receives a positive advantage when its outcome appears less often among the other rollouts than its target probability, and a negative advantage when it appears more often. The size of the advantage is twice the mismatch at that outcome. The factor of two makes the expected update equal to the negative gradient of $\mmd^2$ (Section~\ref{sec:optimizes}).

The witness advantage needs only the outcome counts of the group. We use $A_i$ as the reward $r_i$ in GRPO and leave the rest of GRPO unchanged. GRPO therefore still subtracts the group mean (Equation~\ref{eq:grpoadv}) and trains on $\tilde A_i = A_i - \bar A$, where $\bar A$ is the mean of $A_i$ over the group. Throughout the paper, $A_i$ denotes an advantage before this subtraction and $\tilde A_i$ the advantage after it.

One could instead train on the total variation distance, our error measure (Section~\ref{sec:setup}). It has the same form, $2\,\tv(p,q) = \sup_{|f| \le 1} (\E_{p} f - \E_{q} f)$, and the same minimizer, $\pi_\theta = q$. However, its witness is only the sign of $\pi_\theta - q$. The sign indicates whether a rollout's outcome is too frequent or too rare, but not by how much. Section~\ref{sec:simpler} shows that without the size, training does not reach the target on outcomes with small target probability.

%% file: main_files/kwitnessopt.tex
Recall from Section~\ref{sec:setup} that, ignoring clipping and the KL penalty, a GRPO step moves the parameters along $\tfrac1G \sum_i \tilde A_i \nabla_\theta \log P_\theta(y_i)$. Each response $y_i$ gains log-probability in proportion to $\tilde A_i$. Let $U$ be this direction with $\tilde A_i$ replaced by the witness advantage $A_i$ of Equation~\ref{eq:witness}, and let $U_{\mathrm{c}}$ be the direction GRPO applies, with $\tilde A_i = A_i - \bar A$. Expectations are over the $G$ rollouts of a group, which are independent draws from the model (Appendix~\ref{app:tools}). Proposition~\ref{prop:witness} computes the expected update.

\begin{restatable}[The witness advantage follows the negative gradient of $\mmd^2$]{proposition}{propwitness}
\label{prop:witness}
Let every rollout in a group of $G \ge 2$ rollouts receive the witness advantage. Then $\E[U] = -\nabla_\theta \mmd^2(\pi_\theta, q)$, and $\mmd^2(\pi_\theta,q)=0$ only at $\pi_\theta = q$. After the group mean is subtracted,
\begin{equation}
\E[U_{\mathrm{c}}] = -\tfrac{G-1}{G}\,\nabla_\theta \mmd^2(\pi_\theta,q) + \tfrac{1}{G}\,\nabla_\theta \|\pi_\theta\|^2 ,
\label{eq:centered}
\end{equation}
where $\|\pi_\theta\|^2 = \sum_x \pi_\theta(x)^2$ is the probability that two independent draws from $\pi_\theta$ give the same outcome\proofin{ (proof in Appendix~\ref{app:witness})}.
\end{restatable}

Before the group mean is subtracted, the expected update is the negative gradient of $\mmd^2$ at every group size. In expectation, training is gradient descent on the squared distance between the model's outcome distribution and the target, and this distance is zero only at the target. Subtracting the group mean multiplies the gradient by $(G-1)/G$. It also adds a term of weight $1/G$ that increases $\|\pi_\theta\|^2$, and so pushes the model to concentrate its probability on fewer outcomes.

Because the concentration term has weight only $1/G$, it shifts the minimizer of the training objective only slightly away from the target. If the model can assign any probability to each outcome and $G \ge 3$, the shift is at most $1/(G-2)$ in TV, which is $0.016$ at $G=64$ (Appendix~\ref{app:witnessfixed}).

%% file: main_files/ksimplerrewards.tex
With the witness advantage, rollouts in the same group receive different advantages, because each advantage depends on the rollout's outcome. We call this \emph{per-outcome credit}. Its sign says whether the outcome is under- or over-produced, and its size grows with the mismatch. To test whether training needs advantages that vary across outcomes and reflect the size of the mismatch, we analyze two simpler rewards, each of which removes one of these properties.

\textbf{Group-scalar reward.}
The most direct way to reward a group for matching the target distribution is to score the group by $-\tv(\hat p, q)$, where $\hat p$ is the empirical distribution of its outcomes, and to give this score to every rollout in the group (Section~\ref{sec:intro}). We call this the \emph{group-scalar reward}. It gives GRPO no training signal, because every rollout in the group receives the same reward and subtracting the group mean leaves all advantages at zero (Lemma~\ref{lem:degeneracy}, Appendix~\ref{app:scalar}). To obtain a nonzero signal, we score several groups per prompt and subtract the mean score across groups. Specifically, we sample $256$ rollouts for one prompt and split them into four subgroups of $64$. Every rollout receives the score of its subgroup minus the mean score of the four subgroups.

This repaired reward optimizes a biased objective. In expectation, training with it performs gradient descent on $\E[\tv(\hat p, q)]$, the expected TV between the target and the empirical distribution of a group of $G$ draws from the model (Proposition~\ref{prop:scalar} in Appendix~\ref{app:scalar}). This objective is not the model's own error $\tv(\pi_\theta, q)$. Even a model that samples exactly from $q$ has an expected TV of order $\frac{1}{\sqrt G}\sum_x \sqrt{q(x)(1-q(x))}$ when $q(x) \ge 1/G$ on the support of $q$, because $G$ draws do not reproduce $q$ exactly. The minimizers of the objective are therefore only guaranteed to be within this distance of $q$, and for some targets, the objective is lower at a model that never produces outcomes with $q(x) < 1/G$ than at $q$ itself.

\textbf{Sign witness.}
The \emph{sign witness} keeps per-outcome credit and removes its size. It gives each rollout the sign of its witness advantage, $\mathrm{sign}(q(x_i) - \hat p_{-i}(x_i)) \in \{-1,0,+1\}$. The model then learns whether each outcome is too rare or too frequent among the other rollouts, but not by how much. To state the properties of the sign witness, we compare it with the value it would take if the group estimated $\pi_\theta$ exactly. We call $\mathrm{sign}(q(x) - \pi_\theta(x))$ the \emph{population sign} of outcome $x$. We call $x$ \emph{low-mass} if $0 < q(x) < 1/(G-1)$. Since a leave-one-out frequency is computed from $G-1$ rollouts, $1/(G-1)$ is the smallest nonzero value it can take. Proposition~\ref{prop:sign} shows that the sign witness trains the model toward the target, except on low-mass outcomes.

\begin{restatable}[The sign witness follows the negative gradient of total variation]{proposition}{propsign}
\label{prop:sign}
Let every rollout in a group of $G \ge 2$ rollouts receive the sign witness. Then the following hold\proofin{ (proof in Appendix~\ref{app:sign})}.
\begin{enumerate}[label=(\roman*),leftmargin=*,itemsep=0pt,topsep=1pt]
\item The population sign is the negative of the witness of total variation. If $q(x) \ne \pi_\theta(x)$ for every outcome $x$, the expected update with the population sign is $-\nabla_\theta\, 2\tv(\pi_\theta,q)$.
\item Given that a rollout has an outcome $x$ with $\delta = q(x) - \pi_\theta(x) \ne 0$, its sign witness and the population sign differ with probability at most $\exp(-2(G-1)\delta^2)$.
\item A rollout with a low-mass outcome $x$ receives expected advantage $2(1-\pi_\theta(x))^{G-1} - 1$, independent of $q(x)$. If the model can assign any probability to each outcome, every stationary point of the expected update therefore gives the same probability to all low-mass outcomes that it does not set to zero. In particular, $\pi_\theta = q$ is not a stationary point when two low-mass outcomes have different target probabilities.
\end{enumerate}
\end{restatable}

Parts (i) and (ii) show that, while the model is still far from the target, the sign witness trains it toward the target. With the population sign, the expected update is the negative gradient of twice the TV. The sign computed from the group matches the population sign with high probability, except at outcomes where $q(x)$ and $\pi_\theta(x)$ differ by less than about $1/\sqrt G$.

Part (iii) shows that the sign witness does not train the probability of a low-mass outcome toward its target. The leave-one-out frequency of a low-mass outcome is either $0$, which is below its target probability, or at least $1/(G-1)$, which is above it. The sign therefore depends only on whether the outcome appeared among the other $G-1$ rollouts, not on how much probability the outcome should have. As a result, an outcome with target probability $0.001$ keeps a positive expected advantage until the model assigns it about $\ln 2/(G-1) = 0.011$ at $G=64$, about ten times its target (Appendix~\ref{app:sign}). Section~\ref{sec:synthetic} checks these analyses in training.

%% file: main_files/kimplementationdetails.tex
\textbf{Targets.}
We use six families of discrete distributions from the Spectrum Suite \citep{sorensen2026spectrum}: biased coin, binomial, geometric, Poisson, hypergeometric, and Zipf. Each family consists of $20$ targets that differ in their parameters. The training set has $16$ targets from each family except hypergeometric, $80$ in all. The \emph{unseen parameters} are the four remaining targets of each training family. The \emph{unseen families} are $100$ targets from five families not used in training: hypergeometric, number of empty boxes, discrete triangular, maximum of dice, and logarithmic series (Appendix~\ref{app:setup}). A model cannot know the possible outcomes of a family it has never seen, so the prompts for these targets list the valid outcomes.

\textbf{Metric.}
To evaluate a model on a target, we sample $n$ responses at temperature $1$, with $n=500$ (up to $5{,}000$ for unseen parameters), and compute the TV between $q$ and the empirical distribution $\hat p_n$ of the parsed outcomes. Unlike training (Section~\ref{sec:setup}), evaluation leaves unparseable responses out of $\hat p_n$ and reports their rate separately.

Even a \emph{perfect sampler}, one that draws exactly from $q$, has positive TV between its empirical distribution and $q$ after any finite number of draws. This TV is larger for targets with more outcomes and for smaller $n$, so raw TV values cannot be compared across targets. We therefore report \emph{excess TV}, the measured TV minus the perfect sampler's expected TV at the same $n$, which we estimate by Monte Carlo. An excess TV of $0$ means the model samples as well as a perfect sampler. We also measure general capabilities with MMLU \citep{hendrycks2021measuring} and IFEval \citep{zhou2023instruction}. For each benchmark, we compute a $95\%$ bootstrap confidence interval for the untrained model's score. A trained score is within the noise band if it lies inside this interval.

\textbf{Implementation details.}
All GRPO runs fine-tune Qwen2.5-1.5B-Instruct in TRL on one 48\,GB GPU, with group size $G=64$, $256$ rollouts per step, learning rate $2\times10^{-6}$, KL weight $0.02$, and temperature $1$. A main run on the synthetic distributions has $600$ steps and takes about one GPU-hour (Appendix~\ref{app:training}).

\textbf{Baselines.}
We compare against three kinds of baselines (details in Appendices~\ref{app:training} and~\ref{app:alternatives}). \emph{Simpler rewards} replace only the reward, with the group-scalar reward or the sign witness of Section~\ref{sec:simpler}. \emph{Supervised training} minimizes the cross-entropy $-\sum_x q(x)\log \pi_\theta(x)$ to the stated PMF on the same targets \citep{zhang2024forcing}. Computing this loss requires the model's probability of every outcome the target can produce. \emph{Training-free baselines} change only the prompt or the decoding: verbalized sampling \citep{zhang2025verbalized}, seed conditioning applied only at inference \citep{nagarajan2025dice}, renormalization onto the valid outcomes, and temperature scaling.

%% file: main_files/ksynthfamilies.tex
\begin{table}[t]
\centering
\caption{On the 100 targets from five distribution families never seen in training, the witness advantage has the lowest median excess TV and keeps general capabilities. The oracle temperature is tuned on each target.}
\label{tab:synthetic}
\small
\setlength{\tabcolsep}{10pt}
\setlength{\aboverulesep}{0pt}
\setlength{\belowrulesep}{0pt}
\setlength{\cmidrulekern}{4pt}
\renewcommand{\arraystretch}{1.25}
\begin{tabular}{lccc}
\toprule
 & sampling error & \multicolumn{2}{c}{capability} \\
\cmidrule(lr){2-2}\cmidrule(lr){3-4}
 & excess TV & MMLU & IFEval \\
\midrule
untrained model & 0.589 & 0.584 & 0.512 \\
\midrule
verbalized sampling & 0.518 & 0.584 & 0.512 \\
oracle temperature & 0.478 & 0.584 & 0.512 \\
\midrule
supervised cross-entropy & 0.406 & 0.482 & 0.281 \\
\oursrow witness advantage (ours) & 0.376 & 0.583 & 0.528 \\
\bottomrule
\end{tabular}
\end{table}

Training with the witness advantage makes the model sample much closer to the target distributions, on the training targets and on unseen parameters and unseen families (Table~\ref{tab:synthfull} in Appendix~\ref{app:more}). On the five families never seen in training, it has the lowest error of all compared methods and keeps general capabilities (Table~\ref{tab:synthetic}).

\textbf{Witness advantage fits the training targets and keeps capability.}
Training removes $79\%$ of the excess TV on the non-Zipf training targets, from $0.477$ to $0.101$. MMLU and IFEval stay inside their noise bands, and fewer than $1\%$ of responses are unparseable. The strongest baseline is supervised cross-entropy training on the target PMF. It fits the training targets better than the witness advantage, but it drops MMLU by $10$ points and IFEval by $23$ points relative to the untrained model. The witness advantage changes them by only $0.1$ and $1.6$ points. Most of the remaining error of the witness advantage is on binomial and Poisson targets, where the draws are far more spread out than the target (Appendix~\ref{app:decomposition}).

\textbf{Training-free baselines start from wrong probabilities.}
We evaluate the training-free baselines on the unseen families (Tables~\ref{tab:synthetic} and~\ref{tab:trainingfree}). The best of them, an oracle temperature per target chosen with knowledge of $q$, removes only $19\%$ of the excess TV. These baselines change only the prompt or the decoding, so they start from the untrained model's probabilities over outcomes. These probabilities are wrong even though the model knows the target (Figure~\ref{fig:phenomenon}b). For the $27$ training targets with at most $12$ outcomes, we compute the untrained model's outcome distribution exactly from its log-probabilities. Its median TV to the target is $0.37$, so the error is present before any outcome is drawn and is not sampling noise. Training with the witness advantage lowers this exact TV to $0.09$. A decoding rule can only reshape the wrong probabilities. Temperature can only sharpen or flatten them, and renormalizing onto the valid outcomes does not help ($0.600$ against $0.589$). Verbalized sampling also fails, because the probabilities the model writes out are as wrong as its draws, with a median TV of $0.71$ per parsed list. Appendix~\ref{app:alternatives} gives the other baselines and the larger untrained models.

\textbf{The improvement transfers to unseen parameters and families.}
Training removes $67\%$ of the excess TV on the $16$ unseen-parameter targets outside Zipf, from $0.483$ to $0.161$. Cross-entropy transfers better here ($0.094$), at the capability cost above.

The trained model also samples better from the five families it never saw in training, and removes $36\%$ of their excess TV (Table~\ref{tab:synthetic}). Supervised cross-entropy was trained on the same prompts, so the comparison in the table is like for like. These training prompts, unlike the unseen-family prompts, do not list the valid outcomes. A retrain on prompts that list them still removes $24\%$, more than the $19\%$ of the oracle temperature, and improves every unseen family (Appendix~\ref{app:more}). Most of the improvement does not come from copying the listed outcomes. When we remove the list from the evaluation prompts, three quarters of the improvement remains (Appendix~\ref{app:attribution}).

\textbf{Training needs per-outcome credit of the right size.}
\begin{figure}[t]
\centering
\includegraphics[width=\linewidth]{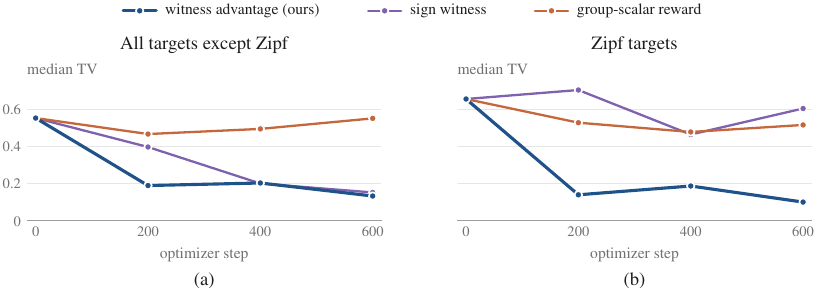}
\caption{Median TV during training on (a)~the 64 training targets outside the Zipf family and (b)~the 16 Zipf targets, which have many low-mass outcomes. Outside Zipf, the sign witness trains as well as the witness advantage, and the group-scalar reward ends where the untrained model starts. On Zipf targets, only the witness advantage brings the TV close to zero.}
\label{fig:training}
\end{figure}
Section~\ref{sec:simpler} shows that the group-scalar reward optimizes a biased objective, and predicts that the sign witness trains like the witness advantage except on low-mass outcomes. With the same data and compute, the group-scalar reward does not lower the median TV on the non-Zipf training targets. It ends at $0.55$, where the untrained model starts (Figure~\ref{fig:training}a), although some families improve (Table~\ref{tab:families}). On unseen parameters, the witness advantage has lower TV than the group-scalar reward in each of three seeds (Table~\ref{tab:paired}). On unseen families the two rewards are closer, and their intervals overlap at every group size in a sweep (Table~\ref{tab:groupsize}).

The sign witness trains as well as the witness advantage on every family except Zipf. Its final median TV on the non-Zipf training targets (Figure~\ref{fig:training}a) is $0.151$, against $0.132$ for the witness advantage. On $36$ unseen targets outside Zipf, its paired difference from the witness advantage is indistinguishable from zero (Table~\ref{tab:paired}). Zipf targets have many low-mass outcomes, those with target probability below $1/(G-1) \approx 0.016$, which the sign witness does not train toward their targets (Proposition~\ref{prop:sign}(iii)). On the Zipf training targets, the median TV ends at $0.60$ with the sign witness and at $0.10$ with the witness advantage (Figure~\ref{fig:training}b). The size of the credit is therefore needed for low-mass outcomes.

%% file: main_files/kopiniondist.tex
\begin{figure}[t]
\centering
\includegraphics[width=\linewidth]{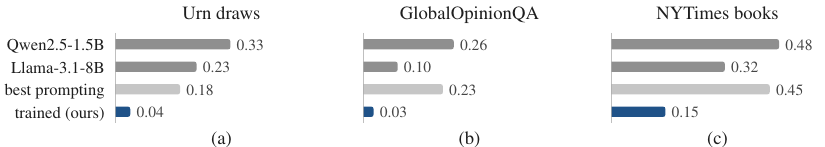}
\caption{Training on stated opinion distributions and urn draws beats prompting and a larger model, and it transfers to a task never seen in training. Bars are median excess TV to the target distribution. (a,~b)~Test prompts of the two training tasks. (c)~The NYTimes task, never seen in training. Best prompting is the better of verbalized sampling and in-context steering, which puts ten draws from the target in the prompt. Trained values are means over three seeds.}
\label{fig:opinions}
\end{figure}

Along with synthetic distributions, we also test our method on tasks closer to the real world, namely opinion distributions and structured outputs, with the same error measure and training procedure. Appendix~\ref{app:real} gives the full setup and results.

\textbf{Training transfers to test prompts and to an unseen task.}
Three tasks from the Spectrum Suite \citep{sorensen2026spectrum} state a distribution over at most six answer options and ask for one draw. Urn draws state color proportions, GlobalOpinionQA \citep{durmus2023global} gives country-level answer distributions to opinion questions, and the NYTimes task gives book preferences by demographic group \citep{meister2025benchmarking}. We train on urn and GlobalOpinionQA prompts and evaluate on test prompts of both (Appendix~\ref{app:real}), on the NYTimes task, never seen in training, and on an implicit variant of the test prompts. The implicit variant removes the stated probabilities, so the model must infer the distribution from the context, such as the balls in the urn or the respondent's country.

Training lowers the error on every evaluation set (Figure~\ref{fig:opinions} and Table~\ref{tab:opinions}). Across three seeds, it removes $80\%$ to $92\%$ of the excess TV on urn and GlobalOpinionQA test prompts, $61\%$ to $72\%$ on the unseen NYTimes task, and $56\%$ to $69\%$ on the implicit variant. It also has lower excess TV than the best prompting method on every set of Figure~\ref{fig:opinions}. On test prompts, the best temperature over $\{0.7, 1.0, 1.3, 1.5\}$ reaches $0.25$ on both tasks, against $0.04$ and $0.03$ after training. On every evaluation set, the trained 1.5B model is also closer to the stated distributions than the untrained Llama-3.1-8B-Instruct. This training has a cost. MMLU falls by $2.1$ points, outside its noise band, while IFEval stays inside its band.

\textbf{Training also improves structured outputs.}
A task adapted from sibling discovery \citep{nagarajan2025dice} names a parent and its children and asks for one pair of siblings, so the model must build the valid pairs before sampling among them. On targets uniform over the pairs, training removes $35\%$ and $39\%$ of the excess TV in two seeds, compared with $5\%$ for the best decoding baseline. On targets that state a non-uniform probability for each pair, it removes $47\%$ (Appendix~\ref{app:structured}).

%% file: main_files/kdiscussion.tex
We show that policy optimization, specifically GRPO, can reduce the mismatch between a language model's stated and sampled distributions, but only when the training signal assigns credit at the level of individual outcomes. A scalar reward for an entire group is insufficient: under GRPO, it is eliminated by group-relative centering. We propose the \textit{witness advantage}, which instead converts the group-level distribution mismatch into per-rollout credit, increasing under-produced outcomes and decreasing over-produced ones. Training with this credit substantially improves sampling on the synthetic training distributions, transfers to unseen parameters and families, and works on tasks closer to the real world, while largely preserving general capabilities.

More broadly, our results suggest that group-defined objectives require more than a group-level score: they require a mechanism for attributing the group's error to individual rollouts. We study this principle for finite outcome spaces, though the witness advantage can be defined for any kernel. An important next step is to apply it to free-form text, where a target distribution may be represented by example responses and a kernel can encode similarity between responses. This would extend distribution-matching post-training to settings in which the outcome space cannot be enumerated.

%% file: main_files/kappendix.tex
\section{Related Work}
\label{app:related}
\label{sec:related}
\paragraph{Sampling failures of language models.}
Language models are poor random number generators \citep{vankoevering2024random,zhang2024forcing,zhao2026dice}, and aligned models do worse than their base models at random number generation and mixed-strategy games \citep{west2025base}. \citet{jang2026cannotsample} show that instruction-tuned models collapse to one answer when asked for a draw, that the collapse widens at every stage of post-training, and that the same models describe the target correctly in one call. \citet{xiao2025flipping} find the same mismatch between describing and sampling a Bernoulli distribution, and \citet{meister2025benchmarking} and \citet{gu2026illusion} find it on opinion distributions and in frontier models. One line of remedies acts at inference time. These methods prompt for a seed string before the draw \citep{misaki2025ssot}, randomize surface features of the prompt across calls \citep{jang2026cannotsample}, or have the model accept or reject proposed draws in natural language, a verbalized form of rejection sampling \citep{xiao2025flipping}. Another line fine-tunes the model with a supervised loss. \citet{zhang2024forcing} minimize the cross-entropy between the model's distribution over valid outputs and the target distribution, and our supervised baseline follows their objective (Appendix~\ref{app:training}). \citet{baldelli2026calibration} fine-tune on prompts that name a distribution family and its parameters. They train either on next-token probabilities computed from the target distribution or on completions sampled from it, and their models sample more accurately than the same models before fine-tuning on families and parameter settings not seen in training. \citet{jang2026cannotsample} name two training-time directions that their evidence calls for, a loss that matches a stated output distribution and rewards scored over groups of samples, and do not run them. We train with policy optimization and a reward of the second kind, and we show that training needs advantages that differ across outcomes and grow with the mismatch at each outcome (Section~\ref{sec:simpler}).

\paragraph{Mode collapse and diversity after alignment.}
Post-training reduces output diversity \citep{kirk2024understanding,padmakumar2024writing,murthy2025onefish}, and language models produce homogeneous outputs \citep{jiang2025hivemind,zhang2025noveltybench}. Proposed causes include typicality bias in preference data \citep{zhang2025verbalized} and sharpening \citep{gai2025smoothing}, the mechanism that \citet{huang2025sharpening} study in self-improvement. Preference learning also implicitly aggregates the differing preferences of annotators \citep{siththaranjan2024distributional}. In reinforcement learning for reasoning, the collapse appears as falling entropy and pass@$k$ \citep{cui2025entropy,he2025unlikely}. Remedies restore diversity without a target \citep{he2025unlikely,li2025darling,li2025choice,nagarajan2025dice,wu2025modeconditioning,springer2026annotations}. Verbalized sampling \citep{zhang2025verbalized} asks for several candidates with probabilities in one call. It is evaluated mainly on diversity and on distributions the model already holds, such as pretraining frequencies, and its only instructed target is a fair die. We evaluate against arbitrary PMFs stated in the prompt. A stated target must counteract collapse while preserving a specific distribution, so diversity restoration is not a substitute. A stated PMF can require less randomness as well as more, as for the coin of Figure~\ref{fig:phenomenon}a.

\paragraph{Distribution matching and group-level rewards in post-training.}
GFlowNet fine-tuning \citep{hu2024amortizing} and distributional control \citep{khalifa2021distributional} fine-tune a model toward a target density over sequences that the task defines through a reward or constraints, and twisted sequential Monte Carlo \citep{zhao2024twisted} samples from such a density at inference time. Distributional control can define the density through moment constraints, such as a stated fraction of outputs with some attribute. In all three methods the task defines the target, and the prompt does not state it as a PMF. DMPO \citep{li2026dmpo} and FlowRL \citep{zhu2025flowrl} match distributions induced by a scalar reward and have no access to a stated PMF. Several recent works make a rollout's GRPO reward depend on its group, and ours is closest to them. GAPO \citep{anschel2025gapo} rewards each rollout by one minus the signed difference between the group frequency of its output and a uniform target over the valid options. For a uniform target, this reward is an affine function of the same signed deviation as our witness advantage, with two differences. Its reward is defined for a uniform target and takes no stated non-uniform PMF, and the frequency counts the rollout itself, the self-count that the leave-one-out frequency removes (Appendix~\ref{app:fullgroup}). Distribution-aware reward \citep{park2026dar} scores a group of numeric rollouts with a proper scoring rule against a scalar label and pays each rollout its leave-one-out marginal contribution to the group score. Its target is one label and its credit is a difference of two group scores. Our target is a stated PMF and our credit is the witness in closed form. The closed form lets us prove unbiasedness and state what mean-centering does. On a binary alphabet, the property-ratio reward of \citet{huang2026alignment} equals the sign witness computed from the full group, up to centering and scale, on every group whose ratio lies outside their tolerance band, and their multi-class variant extends this correspondence to larger alphabets (Appendix~\ref{app:propertyratio}). Pass@$k$ policy optimization \citep{walder2025pkpo} also turns an objective defined on a set of rollouts into per-rollout rewards. It derives unbiased estimators of pass@$k$ and of its gradient, so its objective is the joint success of a set of rollouts and not a match to a distribution. Spectrum Tuning \citep{sorensen2026spectrum} trains for distributional coverage by supervised fine-tuning on samples, and the pluralistic-alignment program \citep{sorensen2024roadmap} names distributional calibration as a goal. GRPO's policy gradient can be written as a U-statistic \citep{zhou2026demystifying}, the same form as the unbiased MMD estimator on which the witness advantage is built.

\section{Notation and Tools for the Proofs}
\label{app:tools}

Appendices~\ref{app:witness}, \ref{app:scalar}, and~\ref{app:sign} prove the results of Section~\ref{sec:method}. This appendix collects what the three proofs share. We fix the notation, show that the proofs may work with outcomes in place of text responses, and state two lemmas about the score function that carry most of the algebra.

\subsection{Notation}

All quantities refer to one fixed prompt, and the following symbols have the same meaning in all three proofs.
\begin{itemize}[leftmargin=*,itemsep=1pt,topsep=2pt]
\item $\mathcal{X}$ is the finite set of outcomes. It contains the invalid outcome $\bot$. The target $q$ is the distribution on $\mathcal{X}$ stated in the prompt, and $q(\bot)=0$.
\item The model generates response $y$ with probability $P_\theta(y)$, where $\theta$ are the model parameters and $\mathcal{Y}$ is the finite set of responses. The parser $\varphi$ maps a response to its outcome $\varphi(y) \in \mathcal{X}$.
\item $\pi_\theta(x) = \sum_{y:\varphi(y)=x} P_\theta(y)$ is the model's outcome distribution (Section~\ref{sec:setup}). It is differentiable in $\theta$, and $\pi_\theta(x) > 0$ for every outcome $x$ (Appendix~\ref{app:reduction} shows why both hold for a language model).
\item A group consists of $G \ge 2$ rollouts. Rollout $i$ has response $y_i$ and outcome $x_i = \varphi(y_i)$, and the outcomes $x_1,\dots,x_G$ are i.i.d.\ from $\pi_\theta$. An expectation $\E$ without a subscript is over the rollouts of a group, or of a batch in Appendix~\ref{app:scalar}.
\item $s(x) = \nabla_\theta \log \pi_\theta(x)$ is the score of outcome $x$, the direction in parameter space that most increases the log-probability of $x$.
\item $\mathbb{1}[\cdot]$ is the indicator of an event. The group frequency $\hat p(x) = \tfrac1G\sum_{j}\mathbb{1}[x_j = x]$ is the fraction of the $G$ rollouts whose outcome is $x$. The leave-one-out frequency $\hat p_{-i}(x) = \tfrac{1}{G-1}\sum_{j \ne i}\mathbb{1}[x_j = x]$ is the fraction among the other $G-1$ rollouts.
\item For a vector $v$ indexed by outcomes, $\|v\| = (\sum_x v(x)^2)^{1/2}$ is the Euclidean norm and $\|v\|_1 = \sum_x |v(x)|$. In particular, $\|\pi_\theta\|^2 = \sum_x \pi_\theta(x)^2$ is the collision probability, the probability that two independent draws from $\pi_\theta$ give the same outcome.
\item $\tv(p,q) = \tfrac12\|p - q\|_1$ is the total variation distance. With the exact-match kernel of Section~\ref{sec:witness}, $\mmd^2(p,q) = \|p - q\|^2$ (Equation~\ref{eq:mmdclosed}). We write $\mmd^2$ without arguments for $\mmd^2(\pi_\theta,q)$.
\item $A_i$ is the advantage of rollout $i$ before GRPO subtracts the group mean $\bar A = \tfrac1G\sum_j A_j$, and $\tilde A_i = A_i - \bar A$ is the centered advantage. Each proof defines its own $A_i$ and its own update. A subscript $\mathrm{c}$, as in $U_{\mathrm{c}}$, marks a quantity that belongs to the centered update.
\item An outcome $x$ is low-mass if $0 < q(x) < 1/(G-1)$ (Section~\ref{sec:simpler}).
\end{itemize}
Every expectation below is a finite sum, so gradients and expectations can be exchanged without further justification.

\subsection{From Responses to Outcomes}
\label{app:reduction}

A GRPO step moves the parameters along $\tfrac1G\sum_i \tilde A_i\, \nabla_\theta \log P_\theta(y_i)$, the score of each response $y_i$ weighted by its advantage (Section~\ref{sec:setup}). The proofs are much simpler if each rollout instead moves along the score $s(x_i)$ of its outcome, because then only the outcome distribution $\pi_\theta$ enters. The two scores are different vectors, so a single update differs between the two views. Lemma~\ref{lem:reduction} shows that the expected updates are equal whenever the weights depend on the responses only through their outcomes. Every advantage in this paper has this property, including the witness advantage and the sign witness before and after centering, and the group-scalar reward.

The lemma uses three facts about our setting. A training response has at most $24$ tokens (Appendix~\ref{app:training}), so the set $\mathcal{Y}$ of responses is finite. The softmax gives every response a probability $P_\theta(y) > 0$ that is differentiable in $\theta$. Every outcome in $\mathcal{X}$, including $\bot$, is the parse of some response, so $\pi_\theta(x) > 0$ and the score $s(x)$ is defined.

\begin{lemma}[Outcome reduction]
\label{lem:reduction}
Let $y_1,\dots,y_N$ be i.i.d.\ from $P_\theta$, let $x_i = \varphi(y_i)$, and let $w_i = a_i(x_1,\dots,x_N)$ be weights given by arbitrary functions $a_i:\mathcal{X}^N \to \mathbb{R}$. Then $x_1,\dots,x_N$ are i.i.d.\ from $\pi_\theta$, and
\begin{equation}
\E\Big[\frac1N\sum_{i=1}^N w_i\, \nabla_\theta \log P_\theta(y_i)\Big] = \E\Big[\frac1N\sum_{i=1}^N w_i\, s(x_i)\Big].
\label{eq:reduction}
\end{equation}
\end{lemma}
\begin{proof}
We prove the two claims in turn.

\textbf{Step 1: the outcomes are i.i.d.} The parser acts on each response separately, so the outcomes are independent, and $\Pr[\varphi(y_i) = x] = \sum_{y:\varphi(y)=x} P_\theta(y) = \pi_\theta(x)$.

\textbf{Step 2: Equation~\ref{eq:reduction}.} Fix $i$ and condition on the other responses $y_{-i}$, whose outcomes we write $x_{-i}$. We write $a_i(x, x_{-i})$ for $a_i$ with $x$ in position $i$. Since $P_\theta(y)\,\nabla_\theta\log P_\theta(y) = \nabla_\theta P_\theta(y)$,
\begin{align*}
\E\big[w_i\, \nabla_\theta \log P_\theta(y_i) \,\big|\, y_{-i}\big] &= \sum_{y \in \mathcal{Y}} a_i\big(\varphi(y), x_{-i}\big)\,\nabla_\theta P_\theta(y) \\
&= \sum_{x \in \mathcal{X}} a_i(x, x_{-i}) \sum_{y:\,\varphi(y)=x} \nabla_\theta P_\theta(y) = \sum_{x \in \mathcal{X}} a_i(x,x_{-i})\,\nabla_\theta \pi_\theta(x).
\end{align*}
The second equality groups the responses by outcome, and the third differentiates the finite sum that defines $\pi_\theta(x)$. Since $\nabla_\theta \pi_\theta(x) = \pi_\theta(x)\, s(x)$, the last expression equals $\E[a_i(x_i,x_{-i})\,s(x_i) \mid x_{-i}]$ with $x_i \sim \pi_\theta$. It depends on $y_{-i}$ only through $x_{-i}$, which are i.i.d.\ from $\pi_\theta$ by Step 1. Taking the expectation over $y_{-i}$ and averaging over $i$ gives Equation~\ref{eq:reduction}.
\end{proof}

Lemma~\ref{lem:reduction} says that when the weights depend only on outcomes, the expected update depends on the model only through its outcome distribution. The weights may be the advantages $A_i$ or the centered advantages $\tilde A_i$. With $N = G$, the lemma covers the expected updates of Propositions~\ref{prop:witness} and~\ref{prop:sign}, before and after centering, including the update with the population sign in Proposition~\ref{prop:sign}(i). With $N = KG$, it covers Proposition~\ref{prop:scalar}(ii) in Appendix~\ref{app:scalar}. The proofs below therefore work with the outcome-level updates. The remaining claims, such as the standard-deviation bound of Proposition~\ref{prop:scalar}(i) and the expected advantage of Proposition~\ref{prop:sign}(iii), concern the values of the advantages. These values are functions of the outcomes, so the claims hold in both views. The lemma is about expectations only. The response score also varies among responses with the same outcome, so the variance of an update differs between the two views.

\subsection{Two Lemmas on the Score}

The first lemma turns the expectation of a function times the score into a sum of gradients of outcome probabilities.

\begin{lemma}[Score identities]
\label{lem:score}
Let $x \sim \pi_\theta$, and let $f:\mathcal{X}\to\mathbb{R}$ be any function. Then
\begin{equation*}
\text{(a)}\quad \E[s(x)] = 0, \qquad\qquad \text{(b)}\quad \E[f(x)\,s(x)] = \sum_x f(x)\,\nabla_\theta \pi_\theta(x).
\end{equation*}
\end{lemma}
\begin{proof}
Since $\pi_\theta(x) > 0$, we have $\pi_\theta(x)\, s(x) = \pi_\theta(x)\,\nabla_\theta \log\pi_\theta(x) = \nabla_\theta \pi_\theta(x)$. Hence
\begin{equation*}
\E[f(x)\,s(x)] = \sum_x \pi_\theta(x)\, f(x)\, s(x) = \sum_x f(x)\, \nabla_\theta \pi_\theta(x),
\end{equation*}
which is (b). For (a), take $f \equiv 1$ and use $\sum_x \nabla_\theta \pi_\theta(x) = \nabla_\theta \sum_x \pi_\theta(x) = \nabla_\theta 1 = 0$.
\end{proof}

The second lemma handles the products that appear when a leave-one-out frequency multiplies a score. Such a product has an indicator that two rollouts match, times the score of one rollout. Only a match that involves the rollout that owns the score contributes.

\begin{lemma}[Pair moments]
\label{lem:pair}
Let $x_1,\dots,x_G$ be i.i.d.\ from $\pi_\theta$, and let $i,j,\ell$ be distinct indices. Then
\begin{align*}
\text{(a)}\quad & \E\big[\mathbb{1}[x_j = x_i]\,s(x_i)\big] = \tfrac12 \nabla_\theta\|\pi_\theta\|^2, \\
\text{(b)}\quad & \E\big[\mathbb{1}[x_\ell = x_j]\,s(x_i)\big] = 0, \\
\text{(c)}\quad & \E\big[q(x_j)\,s(x_i)\big] = 0.
\end{align*}
\end{lemma}
\begin{proof}
For (a), condition on $x_i = x$. The draw $x_j$ equals $x$ with probability $\pi_\theta(x)$, so
\begin{equation*}
\E\big[\mathbb{1}[x_j = x_i]\,s(x_i)\big] = \sum_x \pi_\theta(x)\cdot \pi_\theta(x)\cdot s(x) = \sum_x \pi_\theta(x)\, \nabla_\theta\pi_\theta(x) = \tfrac12 \nabla_\theta \sum_x \pi_\theta(x)^2 .
\end{equation*}
The second equality uses $\pi_\theta(x)\,s(x) = \nabla_\theta\pi_\theta(x)$, the identity behind Lemma~\ref{lem:score}, and the third is the chain rule. For (b) and (c), the first factor does not involve $x_i$, so the expectation factorizes into the expectation of the first factor times $\E[s(x_i)]$, which is zero by Lemma~\ref{lem:score}(a).
\end{proof}

\section{Proofs for the Witness Advantage}
\label{app:witness}

Proposition~\ref{prop:witness} states that the witness advantage makes GRPO's expected update the negative gradient of $\mmd^2$, and computes the bias that subtracting the group mean adds. We restate and prove it. We then show that counting a rollout in its own frequency strengthens the concentration term (Appendix~\ref{app:fullgroup}), find where the centered update moves the model (Appendix~\ref{app:witnessfixed}), work through a small example (Appendix~\ref{app:toy}), and list the numerical checks (Appendix~\ref{app:witnesschecks}).

Rollout $i$ receives the witness advantage $A_i = 2\big(q(x_i) - \hat p_{-i}(x_i)\big)$ of Equation~\ref{eq:witness}, and GRPO trains on the centered advantage $\tilde A_i = A_i - \bar A$. Section~\ref{sec:optimizes} defines the updates $U$ and $U_{\mathrm{c}}$ with the response scores $\nabla_\theta \log P_\theta(y_i)$. In this appendix we reuse the two symbols for the outcome-level updates
\begin{equation*}
U = \frac1G\sum_{i=1}^G A_i\, s(x_i), \qquad U_{\mathrm{c}} = \frac1G\sum_{i=1}^G \tilde A_i\, s(x_i).
\end{equation*}
By Lemma~\ref{lem:reduction}, the two versions of each update have the same expectation, so it suffices to prove the proposition for the outcome-level updates. As fixed in Appendix~\ref{app:tools}, the outcomes are i.i.d.\ from $\pi_\theta$ and $\pi_\theta(x) > 0$ for every $x$.

\begin{restatedbox}
\propwitness*
\end{restatedbox}

\paragraph{Proof idea.}
Given the outcome $x_i$ of rollout $i$, the other $G-1$ rollouts are still independent draws from $\pi_\theta$, so the leave-one-out frequency $\hat p_{-i}(x_i)$ has conditional mean $\pi_\theta(x_i)$. The advantage therefore has conditional mean $2\big(q(x_i) - \pi_\theta(x_i)\big)$, the negative of the witness function up to scale, and Lemma~\ref{lem:score} turns this into $-\nabla_\theta\mmd^2$ (Steps 1 and 2). The group mean $\bar A$ does not have this property. It contains $A_j$ for every $j \ne i$, and $A_j$ depends on $x_i$ because rollout $i$ is one of the rollouts that $\hat p_{-j}$ counts. Step 3 computes this dependence exactly.

\paragraph{Step 1: the MMD in closed form.}
With the exact-match kernel $k(x,x') = \mathbb{1}[x = x']$, the kernel expansion of the squared MMD, with independent draws $x,x' \sim \pi_\theta$ and $z,z' \sim q$, is
\begin{equation*}
\mmd^2(\pi_\theta,q) = \E[k(x,x')] + \E[k(z,z')] - 2\,\E[k(x,z)].
\end{equation*}
Each term is a collision probability, the probability that two independent draws give the same outcome:
\begin{equation*}
\E[k(x,x')] = \sum_x \pi_\theta(x)^2, \qquad \E[k(z,z')] = \sum_x q(x)^2, \qquad \E[k(x,z)] = \sum_x \pi_\theta(x)\,q(x).
\end{equation*}
Substituting the three terms gives
\begin{equation}
\mmd^2(\pi_\theta,q) = \sum_x \pi_\theta(x)^2 + \sum_x q(x)^2 - 2\sum_x \pi_\theta(x)\,q(x) = \sum_x \big(\pi_\theta(x) - q(x)\big)^2 .
\label{eq:mmdclosed}
\end{equation}
A sum of squares vanishes only when every term vanishes, so $\mmd^2(\pi_\theta,q) = 0$ only at $\pi_\theta = q$, which is the second claim of the proposition. Differentiating Equation~\ref{eq:mmdclosed} gives the gradient that the update should match,
\begin{equation}
-\nabla_\theta \mmd^2(\pi_\theta,q) = 2\sum_x \big(q(x) - \pi_\theta(x)\big)\,\nabla_\theta \pi_\theta(x).
\label{eq:mmdgrad}
\end{equation}

\paragraph{Step 2: the update before centering is unbiased.}
The rollouts are exchangeable, so $\E[U] = \E[A_1\, s(x_1)]$. We split $A_1 = 2q(x_1) - 2\hat p_{-1}(x_1)$ and treat the two terms separately. The target term is Lemma~\ref{lem:score}(b) with $f = 2q$,
\begin{equation*}
\E\big[2q(x_1)\,s(x_1)\big] = 2\sum_x q(x)\,\nabla_\theta\pi_\theta(x).
\end{equation*}
The leave-one-out term is an average of $G-1$ indicator terms, each given by Lemma~\ref{lem:pair}(a),
\begin{align*}
\E\big[2\hat p_{-1}(x_1)\,s(x_1)\big] &= \frac{2}{G-1}\sum_{j=2}^{G} \E\big[\mathbb{1}[x_j = x_1]\,s(x_1)\big] \\
&= \frac{2}{G-1}\cdot (G-1)\cdot \tfrac12\nabla_\theta\|\pi_\theta\|^2 = 2\sum_x \pi_\theta(x)\,\nabla_\theta\pi_\theta(x).
\end{align*}
Subtracting the second display from the first and comparing with Equation~\ref{eq:mmdgrad},
\begin{equation*}
\E[U] = 2\sum_x \big(q(x) - \pi_\theta(x)\big)\,\nabla_\theta\pi_\theta(x) = -\nabla_\theta \mmd^2(\pi_\theta,q).
\end{equation*}
This is the first claim of the proposition. It holds at every $G \ge 2$.

\paragraph{Step 3: the update after centering.}
Centering replaces every advantage $A_i$ by $\tilde A_i = A_i - \bar A$. By linearity,
\begin{equation}
U_{\mathrm{c}} = U - \bar A\,\bar g, \qquad \text{where } \bar g = \frac1G\sum_{i=1}^G s(x_i).
\label{eq:ucdecomp}
\end{equation}
It remains to compute $\E[\bar A\,\bar g]$, which is an average of $G^2$ products,
\begin{equation}
\E[\bar A\,\bar g] = \frac{1}{G^2}\sum_{i=1}^{G}\sum_{j=1}^{G}\E\big[A_j\,s(x_i)\big].
\label{eq:pairsum}
\end{equation}
The $G^2$ terms fall into two kinds.
\begin{itemize}[leftmargin=*,itemsep=1pt,topsep=2pt]
\item Diagonal terms ($i = j$). There are $G$ of them, and each equals $\E[A_i\,s(x_i)] = -\nabla_\theta\mmd^2$ by Step 2.
\item Off-diagonal terms ($i \ne j$). There are $G(G-1)$ of them. Split $A_j = 2q(x_j) - 2\hat p_{-j}(x_j)$ again. The target part vanishes by Lemma~\ref{lem:pair}(c). The leave-one-out part is
\begin{equation*}
\hat p_{-j}(x_j) = \frac{1}{G-1}\sum_{\ell \ne j}\mathbb{1}[x_\ell = x_j].
\end{equation*}
Every index $\ell \notin \{i,j\}$ gives a match that does not involve rollout $i$, so it contributes zero by Lemma~\ref{lem:pair}(b). The single index $\ell = i$ contributes $\tfrac12\nabla_\theta\|\pi_\theta\|^2$ by Lemma~\ref{lem:pair}(a).
\end{itemize}
Hence, for $i \ne j$,
\begin{equation}
\E\big[A_j\,s(x_i)\big] = -\frac{2}{G-1}\cdot\frac12\nabla_\theta\|\pi_\theta\|^2 = -\frac{1}{G-1}\nabla_\theta\|\pi_\theta\|^2 .
\label{eq:offdiag}
\end{equation}
Substituting both kinds of term into Equation~\ref{eq:pairsum},
\begin{align*}
\E[\bar A\,\bar g] &= \frac{1}{G^2}\Big[ G\cdot\big(-\nabla_\theta\mmd^2\big) + G(G-1)\cdot\Big(-\frac{1}{G-1}\nabla_\theta\|\pi_\theta\|^2\Big)\Big] \\
&= -\frac1G\nabla_\theta\mmd^2 - \frac1G\nabla_\theta\|\pi_\theta\|^2 .
\end{align*}
Finally, Equation~\ref{eq:ucdecomp} and Step 2 give
\begin{align*}
\E[U_{\mathrm{c}}] = \E[U] - \E[\bar A\,\bar g] &= -\nabla_\theta\mmd^2 + \frac1G\nabla_\theta\mmd^2 + \frac1G\nabla_\theta\|\pi_\theta\|^2 \\
&= -\frac{G-1}{G}\nabla_\theta\mmd^2 + \frac1G\nabla_\theta\|\pi_\theta\|^2,
\end{align*}
which is Equation~\ref{eq:centered}. This completes the proof. \qed

\paragraph{Why the group mean biases the update.}
The standard argument that subtracting a baseline leaves the expected update unchanged requires the baseline to be independent of the rollout's own outcome. A constant baseline satisfies this condition. The group mean $\bar A$ does not, and Equation~\ref{eq:offdiag} gives the exact size of the dependence.

\subsection{Counting the Rollout Itself}
\label{app:fullgroup}

Section~\ref{sec:witness} builds the witness advantage from the leave-one-out frequency and not from the group frequency $\hat p(x_i)$, which counts rollout $i$ itself. We now show that the self-count adds a term that concentrates the model, and that after centering this term has about twice the weight of the term that centering adds to the leave-one-out update. The full-group advantage and its update are
\begin{equation*}
\hat A_i = 2\big(q(x_i) - \hat p(x_i)\big), \qquad \hat U = \frac1G\sum_{i=1}^G \hat A_i\, s(x_i).
\end{equation*}
As in Step 2, $\E[\hat U] = \E[\hat A_1\,s(x_1)]$, and the target term is unchanged. The frequency term differs because rollout $1$ always matches itself:
\begin{equation*}
\hat p(x_1) = \frac1G\Big(1 + \sum_{j \ge 2}\mathbb{1}[x_j = x_1]\Big).
\end{equation*}
The constant $1$ contributes zero by Lemma~\ref{lem:score}(a), and each of the $G-1$ indicators contributes $\tfrac12\nabla_\theta\|\pi_\theta\|^2$ by Lemma~\ref{lem:pair}(a):
\begin{equation*}
\E\big[2\hat p(x_1)\,s(x_1)\big] = \frac2G\,\E[s(x_1)] + \frac{2(G-1)}{G}\cdot\tfrac12\nabla_\theta\|\pi_\theta\|^2 = \frac{G-1}{G}\nabla_\theta\|\pi_\theta\|^2 .
\end{equation*}
Using $\nabla_\theta\|\pi_\theta\|^2 = 2\sum_x \pi_\theta(x)\nabla_\theta\pi_\theta(x)$ and Equation~\ref{eq:mmdgrad},
\begin{equation*}
\E[\hat U] = 2\sum_x q(x)\nabla_\theta\pi_\theta(x) - \frac{G-1}{G}\nabla_\theta\|\pi_\theta\|^2 = -\nabla_\theta\mmd^2(\pi_\theta,q) + \frac1G\nabla_\theta\|\pi_\theta\|^2 .
\end{equation*}
Before centering, the self-count adds a term of weight $1/G$ that increases the collision probability and so pushes the model to concentrate its probability on fewer outcomes.

Centering adds a term of the same form to the leave-one-out update (Equation~\ref{eq:centered}), so the comparison that matters is between the two centered updates. Let $\hat U_{\mathrm{c}} = \frac1G\sum_i (\hat A_i - \bar{\hat A})\,s(x_i)$, where $\bar{\hat A}$ is the group mean of the $\hat A_i$. We repeat Step 3 with $\hat A_j$ in place of $A_j$. The diagonal terms equal $\E[\hat U]$. For an off-diagonal term ($i \ne j$), the frequency $\hat p(x_j)$ contains the constant $\tfrac1G$ from the self-match of rollout $j$, which contributes zero by Lemma~\ref{lem:score}(a), and a single match with $x_i$, which contributes $\tfrac1G\cdot\tfrac12\nabla_\theta\|\pi_\theta\|^2$ by Lemma~\ref{lem:pair}(a). Hence $\E[\hat A_j\,s(x_i)] = -\tfrac1G\nabla_\theta\|\pi_\theta\|^2$, and
\begin{equation*}
\E[\hat U_{\mathrm{c}}] = -\frac{G-1}{G}\nabla_\theta\mmd^2(\pi_\theta,q) + \frac{2(G-1)}{G^2}\nabla_\theta\|\pi_\theta\|^2 .
\end{equation*}
After centering, the self-count raises the weight of the concentration term from $1/G$ to $2(G-1)/G^2$, nearly twice as much. For a distribution $\pi$, completing the square gives $\frac{G-1}{G}(\pi(x)-q(x))^2 - \frac{2(G-1)}{G^2}\pi(x)^2 = \frac{(G-1)(G-2)}{G^2}\big(\pi(x) - \tfrac{G}{G-2}q(x)\big)^2 + \text{const}$. The argument of Appendix~\ref{app:witnessfixed}, with $G/(G-2)$ in place of $\gamma$, then bounds the shift of the minimizer by $2/(G-2)$ in TV instead of $1/(G-2)$.

Removing the self-count is the U-statistic correction of the MMD estimator \citep{gretton2012kernel,binkowski2018demystifying}. It is also how the leave-one-out baseline for REINFORCE \citep{williams1992simple} is built, from the other rollouts only \citep{kool2019buy,ahmadian2024back}.

\subsection{Where the Centered Update Moves the Model}
\label{app:witnessfixed}

Centering adds a term of weight $1/G$ to the expected update (Proposition~\ref{prop:witness}). We now ask how far this term moves the distribution that training approaches. For a distribution $\pi$ on $\mathcal{X}$, define the loss
\begin{equation*}
F_{\mathrm{c}}(\pi) = \frac{G-1}{G}\,\mmd^2(\pi,q) - \frac1G\,\|\pi\|^2 .
\end{equation*}
By Equation~\ref{eq:centered}, the expected centered update is $\E[U_{\mathrm{c}}] = -\nabla_\theta F_{\mathrm{c}}(\pi_\theta)$, so in expectation training descends $F_{\mathrm{c}}$. This loss is uniformly close to $\mmd^2$. Since $\mmd^2(\pi,q) = \|\pi - q\|^2 \le \|\pi - q\|_1 \le 2$ and $\|\pi\|^2 \le 1$,
\begin{equation*}
\big|F_{\mathrm{c}}(\pi) - \mmd^2(\pi,q)\big| = \frac1G\big(\mmd^2(\pi,q) + \|\pi\|^2\big) \le \frac3G .
\end{equation*}
A small difference between two losses does not by itself bound the distance between their minimizers, so we compute the minimizer of $F_{\mathrm{c}}$ directly.

\paragraph{The minimizer.}
Suppose that the model can assign any probability to each outcome. At $G = 2$, the centered advantages are $\tilde A_1 = q(x_1) - q(x_2) = -\tilde A_2$, so the frequencies cancel. The loss $F_{\mathrm{c}}(\pi) = \tfrac12\|q\|^2 - \sum_x \pi(x)\,q(x)$ is then linear in $\pi$, and its minimizer is a point mass on a most likely outcome of $q$. For $G \ge 3$, we find the distribution $\pi^\star$ that minimizes $F_{\mathrm{c}}$ over all distributions on $\mathcal{X}$. Let $\gamma = (G-1)/(G-2)$, which is slightly larger than $1$. Completing the square for each outcome $x$ gives
\begin{equation*}
\frac{G-1}{G}\big(\pi(x) - q(x)\big)^2 - \frac1G\,\pi(x)^2 = \frac{G-2}{G}\big(\pi(x) - \gamma\,q(x)\big)^2 + (\text{a term that does not depend on } \pi).
\end{equation*}
Summed over $x$, $F_{\mathrm{c}}$ is a positive multiple of the squared Euclidean distance from $\pi$ to the vector $\gamma q$, plus a constant. Its minimizer $\pi^\star$ is therefore unique, and it is the Euclidean projection of $\gamma q$ onto the set of distributions on $\mathcal{X}$. We write this projection as minimizing $\|\pi - \gamma q\|^2$ subject to $\sum_x \pi(x) = 1$ and $\pi(x) \ge 0$. The Karush--Kuhn--Tucker conditions, with multiplier $2\tau$ on the equality constraint, give
\begin{equation}
\pi^\star(x) = \max\big(\gamma\,q(x) - \tau,\ 0\big) \quad\text{for all } x .
\label{eq:projection}
\end{equation}
Here $\tau$ is the number for which $\sum_x \pi^\star(x) = 1$. This number is unique. The sum $\sum_x \max(\gamma q(x) - \tau, 0)$ is continuous and nonincreasing in $\tau$, and strictly decreasing wherever it is positive. It equals $\gamma > 1$ at $\tau = 0$ and $0$ at $\tau = \gamma \max_x q(x)$, so it takes the value $1$ at exactly one $\tau$.

Equation~\ref{eq:projection} implies three properties of $\pi^\star$. It gives zero probability to every outcome that $q$ excludes, it has a closed form when no target probability is too small, and it is within $1/(G-2)$ of $q$ in TV. We prove them in turn.

First, $\tau > 0$. If $\tau \le 0$, Equation~\ref{eq:projection} would give $\sum_x \pi^\star(x) \ge \gamma\sum_x q(x) = \gamma > 1$. Every outcome with $q(x) = 0$, including $\bot$, therefore has $\pi^\star(x) = 0$.

Second, $\pi^\star$ has a closed form when no target probability is too small. Let $S = \{x : q(x) > 0\}$ be the support of $q$, let $u$ be the uniform distribution on $S$, and suppose that $(G-1)\,q(x) \ge 1/|S|$ for every $x \in S$. Then
\begin{equation*}
\tau = \frac{1}{(G-2)\,|S|} .
\end{equation*}
To check this value, note that it satisfies $\gamma\,q(x) - \tau \ge 0$ on $S$ and $\sum_{x \in S}(\gamma\,q(x) - \tau) = \gamma - 1/(G-2) = 1$. It is therefore the unique $\tau$ of Equation~\ref{eq:projection}. Substituting it into Equation~\ref{eq:projection} gives, on $S$,
\begin{equation*}
\pi^\star = \frac{(G-1)\,q - u}{G-2}, \qquad \pi^\star - q = \frac{q - u}{G-2}.
\end{equation*}
Each outcome moves away from uniform by $1/(G-2)$ times the difference between its target probability and $1/|S|$, which is less than $2\%$ of that difference at $G = 64$. If some $x \in S$ has $(G-1)\,q(x) < 1/|S|$, this formula would give $\pi^\star(x) < 0$. Equation~\ref{eq:projection} instead sets to zero every outcome with $\gamma\,q(x) \le \tau$, which are the outcomes with the smallest target probabilities.

Third, $\pi^\star$ is always close to $q$. Let $S^\star = \{x : \pi^\star(x) > 0\}$ be its support, write $q(B) = \sum_{x \in B} q(x)$ for the target probability of a set $B$ of outcomes, and write $z^+ = \max(z, 0)$. The outcomes that $\pi^\star$ sets to zero carry little target probability. Summing Equation~\ref{eq:projection} over $S^\star$ gives $\gamma\,q(S^\star) - |S^\star|\,\tau = 1$. Since $\tau > 0$ by the first property, $q(S^\star) \ge 1/\gamma$, so
\begin{equation*}
1 - q(S^\star) \le 1 - \frac1\gamma = \frac{1}{G-1}.
\end{equation*}
The outcomes that $\pi^\star$ keeps gain little probability. Only outcomes in $S^\star$ can have $\pi^\star(x) > q(x)$, and there $\pi^\star(x) - q(x) = q(x)/(G-2) - \tau < q(x)/(G-2)$. Hence
\begin{equation*}
\tv(\pi^\star, q) = \sum_x \big(\pi^\star(x) - q(x)\big)^+ \le \frac{q(S^\star)}{G-2} \le \frac{1}{G-2},
\end{equation*}
which is $0.016$ at $G = 64$.

\paragraph{What the minimizer means for training.}
The distribution $\pi^\star$ is the unique minimizer of the loss that the centered update descends in expectation, and it lies within $1/(G-2)$ of the target in TV. Like $q$ itself, it gives zero probability to $\bot$, so a softmax model approaches it without reaching it. We do not prove that the idealized flow of the centered update converges to $\pi^\star$. In a separate small simulation, we trained a softmax model on the expected centered update for $45$ targets (five group sizes, three alphabet sizes, and random, Zipf, and small-mass targets). It came within $3 \times 10^{-3}$ of $\pi^\star$ in every coordinate, and the remaining distance kept shrinking with more steps.

\subsection{A Worked Example}
\label{app:toy}

Table~\ref{tab:toygroup} evaluates the witness advantage on a group small enough to check by hand. The group has $G = 6$ rollouts with outcomes $(a,a,a,b,c,\bot)$, and the target is $q = (1/2,\, 3/10,\, 1/5,\, 0)$ on $(a,b,c,\bot)$. For a rollout on $a$, the other five rollouts contain two more copies of $a$, so $\hat p_{-i}(a) = 2/5$ and $A_i = 2(1/2 - 2/5) = 1/5$. For the rollout on $b$, none of the other five is $b$, so $A_i = 2(3/10 - 0) = 3/5$. Likewise $A_i = 2/5$ for $c$. For the lone $\bot$, both $q(\bot)$ and $\hat p_{-i}(\bot)$ are zero, so $A_i = 0$.

The full-group variant uses $\hat p(a) = 3/6$ and $\hat p(b) = \hat p(c) = \hat p(\bot) = 1/6$. It gives $0$ to $a$, $4/15$ to $b$, $1/15$ to $c$, and $-1/3$ to $\bot$. The self-count changes every value. For example, the three copies of $a$ match the target probability exactly, so the full-group variant gives them zero advantage, although among the other five rollouts $a$ is under-produced.

\begin{table}[h]
\centering
\caption{The witness advantage and its full-group variant on a group of six rollouts with outcomes $(a,a,a,b,c,\bot)$, in exact fractions.}
\label{tab:toygroup}
\small
\setlength{\tabcolsep}{9pt}
\renewcommand{\arraystretch}{1.12}
\begin{tabular}{@{}lcccc@{}}
\toprule
 & $a$ \tsub{(3 copies)} & $b$ & $c$ & $\bot$ \\
\hmidrule
target probability $q(x)$ & $1/2$ & $3/10$ & $1/5$ & $0$ \\
witness advantage $A_i$ (ours) & $1/5$ & $3/5$ & $2/5$ & $0$ \\
full-group variant $\hat A_i$ & $0$ & $4/15$ & $1/15$ & $-1/3$ \\
\bottomrule
\end{tabular}
\end{table}

\paragraph{Invalid outcomes.}
The example shows that a lone invalid rollout receives advantage zero. In general, a group with $m \ge 2$ invalid rollouts gives each of them $A_i = -2(m-1)/(G-1)$. The model is still pushed away from invalid outputs in expectation, because the conditional mean of the advantage at $\bot$ is $2\big(q(\bot) - \pi_\theta(\bot)\big) = -2\pi_\theta(\bot) < 0$.

\subsection{Numerical Checks}
\label{app:witnesschecks}

We check Proposition~\ref{prop:witness} numerically. We recompute Table~\ref{tab:toygroup} in exact fractions and check the expected update against the automatic gradient of $-\mmd^2$ (largest deviation $1.5\times10^{-8}$). We also enumerate every group for $|\mathcal{X}| = 3$ and small $G$, which confirms Equation~\ref{eq:centered} to $3.4\times10^{-15}$, including an alphabet with a zero-mass outcome. The same enumeration confirms the centered full-group update of Appendix~\ref{app:fullgroup}. A Monte Carlo run with $40{,}000$ groups at $G = 8$ shows the bias of the full-group update. Its largest coordinate-wise deviation from $-\nabla_\theta\mmd^2$ is $0.016$, against $0.001$ for the leave-one-out update, a deviation of the size of the Monte Carlo error.

\section{Proofs for the Group-Scalar Reward}
\label{app:scalar}

Section~\ref{sec:simpler} makes two claims about the group-scalar reward, which scores a group by the negative TV between its empirical distribution and the target and gives this score to every rollout. First, with one group per prompt it gives GRPO no training signal. Second, the repaired version, which scores several subgroups and subtracts their mean score, optimizes a biased objective. This appendix states both claims formally and proves them. It then simulates when the bias starts to matter.

\subsection{Setting and Statements}

The notation of Appendix~\ref{app:tools} remains in force. In addition:
\begin{itemize}[leftmargin=*,itemsep=1pt,topsep=2pt]
\item A batch for one prompt consists of $K \ge 2$ subgroups, each of $G$ rollouts drawn i.i.d.\ from $\pi_\theta$. We trained with $K = 4$ and $G = 64$.
\item Subgroup $\ell$ has empirical distribution $\hat p^{(\ell)}$ and TV $T_\ell = \tv(\hat p^{(\ell)}, q)$. The mean TV of the batch is $\bar T = \tfrac1K\sum_{\ell=1}^K T_\ell$.
\item Every rollout in subgroup $\ell$ receives the reward $-T_\ell$. Subtracting the mean reward over the batch gives rollout $i$ the centered advantage $\tilde A_i = \bar T - T_{\ell(i)}$, where $\ell(i)$ is its subgroup. The outcome-level update is $U = \tfrac{1}{KG}\sum_i \tilde A_i\, s(x_i)$. This reward is used only with centering, so we write its centered update as $U$ without the subscript $\mathrm{c}$.
\item $\hat p_G$ is the empirical distribution of $G$ i.i.d.\ draws. It is the group frequency $\hat p$ of Appendix~\ref{app:tools}, with the group size made explicit. For a distribution $\pi$ on $\mathcal{X}$, the surrogate loss $F_G(\pi) = \E_{\pi}[\tv(\hat p_G, q)]$ is the expected TV of $G$ draws from $\pi$. Its bias $B_G(\pi) = F_G(\pi) - \tv(\pi,q)$ is the amount by which the expected TV of $G$ draws exceeds the TV of $\pi$ itself.
\end{itemize}

The first result explains why the group-scalar reward needs subgroups.

\begin{lemma}[Equal rewards give zero advantages]
\label{lem:degeneracy}
Let every rollout in a GRPO group receive the same reward. Then every advantage in that group is zero after the group mean is subtracted.
\end{lemma}
\begin{proof}
If every rollout receives the reward $r_0$, then every centered advantage is $r_0 - r_0 = 0$. The update is therefore exactly zero on every batch. The original GRPO also divides by the standard deviation of the rewards. This standard deviation is zero, and implementations add a small constant to the denominator, so the quotient is still zero.
\end{proof}
In a ten-step pilot run of this reward, the logged gradient norm was zero at every step.

The second result says what the repaired reward optimizes and why this objective is biased.

\begin{proposition}[The group-scalar reward optimizes a biased objective]
\label{prop:scalar}
Let a batch contain $K \ge 2$ subgroups of $G$ rollouts, and give every rollout in subgroup $\ell$ the centered advantage $\bar T - T_\ell$. Then the following hold (proof in Appendix~\ref{app:scalarproofs}).
\begin{enumerate}[label=(\roman*),leftmargin=*,itemsep=0pt,topsep=1pt]
\item The advantages take at most $K$ distinct values per batch, and each advantage has standard deviation at most $\sqrt{(K-1)/K}\,/\,(2\sqrt G)$.
\item The expected update is $\E[U] = -\frac{K-1}{KG}\,\nabla_\theta F_G(\pi_\theta)$.
\item Let $\mathcal{X}_1 = \{x : 1/G \le q(x) \le 1 - 1/G\}$ be the outcomes whose target probability $G$ draws can resolve. At $\pi_\theta = q$,
\begin{equation*}
\frac{1}{2\sqrt{2G}}\sum_{x \in \mathcal{X}_1}\sqrt{q(x)\big(1-q(x)\big)} \;\le\; B_G(q) \;\le\; \frac{1}{2\sqrt{G}}\sum_{x \in \mathcal{X}}\sqrt{q(x)\big(1-q(x)\big)} .
\end{equation*}
\item Every minimizer of $F_G$ over all distributions on $\mathcal{X}$ is within TV $B_G(q)$ of $q$. For some targets, $F_G$ is smaller at a distribution that gives zero probability to every outcome with $q(x) < 1/G$ than at $q$ itself.
\end{enumerate}
\end{proposition}

Part (i) limits how much the training signal can vary within a batch. At $K = 4$ and $G = 64$ the bound is $0.054$, and Appendix~\ref{app:decomposition} measures $0.031$. Part (ii) says that, in expectation, the repaired reward performs gradient descent on the surrogate loss $F_G$, the expected TV of $G$ draws, and not on the model's own TV. The factor $(K-1)/(KG) \approx 0.012$ matters little, because AdamW is nearly invariant to a constant rescaling of the gradient. Parts (iii) and (iv) concern the surrogate near the target. Even a model that samples exactly from $q$ has $F_G(q) = B_G(q) > 0$, so the surrogate penalizes the target itself by the sampling error of $G$ draws. Suppose every outcome in the support of $q$ has $q(x) \ge 1/G$. If the support has at least two outcomes, each of them then also has $q(x) \le 1 - 1/G$, so the support lies in $\mathcal{X}_1$ and the two sides of part (iii) differ by a factor of $\sqrt2$. The bias is then of order $\frac{1}{\sqrt G}\sum_x\sqrt{q(x)(1-q(x))}$, the rate quoted in Section~\ref{sec:simpler}. When many outcomes have target probability below $1/G$, the surrogate can prefer to give them zero probability.

\subsection{Proofs}
\label{app:scalarproofs}

The bound of part (i) rests on the following lemma.

\begin{lemma}[The subgroup TV varies little]
\label{lem:dispersion}
Let $T_\ell$ be the TV of a subgroup of $G$ i.i.d.\ draws. Then changing one draw changes $T_\ell$ by at most $1/G$, and $\mathrm{Var}(T_\ell) \le 1/(4G)$.
\end{lemma}
\begin{proof}
Replacing one draw moves $\hat p^{(\ell)}$ by $-1/G$ in one coordinate and by $+1/G$ in another. Since $T_\ell = \tfrac12\|\hat p^{(\ell)} - q\|_1$, it moves by at most $\tfrac12 \cdot \tfrac2G = \tfrac1G$. The Efron--Stein inequality bounds the variance of a function of independent draws by the sum over draws of the expected variance over that draw with the others fixed. With the others fixed, $T_\ell$ lies in an interval of length at most $1/G$, so its variance over one draw is at most $1/(4G^2)$. Summing over the $G$ draws gives $\mathrm{Var}(T_\ell) \le 1/(4G)$.
\end{proof}

\paragraph{Proof of part (i).}
The advantage $\tilde A_i = \bar T - T_{\ell(i)}$ depends on $i$ only through its subgroup, so it takes at most $K$ distinct values in a batch. For the bound, write
\begin{equation*}
\bar T - T_\ell = \frac1K\sum_{\ell' \ne \ell} T_{\ell'} - \frac{K-1}{K}\,T_\ell .
\end{equation*}
The TVs of different subgroups are i.i.d., so the two parts are independent and
\begin{equation*}
\mathrm{Var}(\bar T - T_\ell) = \frac{K-1}{K^2}\mathrm{Var}(T_\ell) + \frac{(K-1)^2}{K^2}\mathrm{Var}(T_\ell) = \frac{K-1}{K}\,\mathrm{Var}(T_\ell) \le \frac{K-1}{4KG}
\end{equation*}
by Lemma~\ref{lem:dispersion}. Taking the square root gives the bound.

The same number also bounds the expected spread of the advantages within one batch. The advantages have mean zero, so the expected within-batch variance is
\begin{equation*}
\E\Big[\frac1K\sum_{\ell=1}^K (\bar T - T_\ell)^2\Big] = \mathrm{Var}(\bar T - T_\ell).
\end{equation*}
By Jensen's inequality, the expected within-batch standard deviation is at most the square root of this value. \qed

\paragraph{Proof of part (ii).}
We write the update as a sum over subgroups and show that only the correlation between a subgroup's TV and its own summed score survives. Let $S_\ell = \sum_{i:\,\ell(i) = \ell} s(x_i)$ be the summed score of subgroup $\ell$. Three facts are needed.
\begin{enumerate}[leftmargin=*,itemsep=1pt,topsep=2pt]
\item $\E[S_\ell] = 0$, by Lemma~\ref{lem:score}(a) applied to each draw.
\item $\E[T_{\ell'}\,S_\ell] = \E[T_{\ell'}]\,\E[S_\ell] = 0$ for $\ell' \ne \ell$, because different subgroups are independent.
\item $\E[T_\ell\,S_\ell] = \nabla_\theta\,\E[T_\ell] = \nabla_\theta F_G(\pi_\theta)$. This is the score-function identity for a function of $G$ i.i.d.\ draws, obtained by differentiating the finite sum $\E[T_\ell] = \sum_{x_1,\dots,x_G} T_\ell(x_1,\dots,x_G)\prod_{i} \pi_\theta(x_i)$.
\end{enumerate}
Grouping the update by subgroup and substituting $\tilde A_i = \bar T - T_{\ell(i)}$,
\begin{align*}
\E[U] &= \frac{1}{KG}\sum_{\ell=1}^{K}\E\big[(\bar T - T_\ell)\,S_\ell\big]
= \frac{1}{KG}\sum_{\ell=1}^{K}\Big(\frac1K\sum_{\ell'=1}^{K}\E[T_{\ell'}\,S_\ell] - \E[T_\ell\,S_\ell]\Big) \\
&= -\frac{1}{KG}\cdot\frac{K-1}{K}\sum_{\ell=1}^{K}\E[T_\ell\,S_\ell],
\end{align*}
where the last step uses fact 2 to keep only the term $\ell' = \ell$ of the inner sum. By fact 3, each of the $K$ terms of the remaining sum equals $\nabla_\theta F_G(\pi_\theta)$, so the sum is $K\,\nabla_\theta F_G(\pi_\theta)$. Substituting,
\begin{equation*}
\E[U] = -\frac{K-1}{KG}\,\nabla_\theta F_G(\pi_\theta).
\end{equation*}
\qed

\paragraph{Proof of part (iii).}
Both bounds rest on the fact that $G\hat p_G(x)$ is a binomial count with $G$ trials and success probability $\pi_\theta(x)$.

For the upper bound, fix any $\pi_\theta$. For each outcome $x$, the triangle inequality and then the Cauchy--Schwarz inequality give
\begin{equation*}
\E\big[|\hat p_G(x) - q(x)|\big] \le \big|\pi_\theta(x) - q(x)\big| + \E\big[|\hat p_G(x) - \pi_\theta(x)|\big]
\le \big|\pi_\theta(x) - q(x)\big| + \sqrt{\frac{\pi_\theta(x)\big(1-\pi_\theta(x)\big)}{G}},
\end{equation*}
where the second step uses that the binomial count has variance $G\pi_\theta(x)(1-\pi_\theta(x))$. Summing over $x$ and halving gives
\begin{equation*}
B_G(\pi_\theta) \le \frac{1}{2\sqrt G}\sum_x \sqrt{\pi_\theta(x)\big(1-\pi_\theta(x)\big)},
\end{equation*}
which at $\pi_\theta = q$ is the upper bound.

For the lower bound, take $\pi_\theta = q$. Since $\tv(q,q) = 0$,
\begin{equation*}
B_G(q) = F_G(q) = \frac{1}{2G}\sum_x \E\big[|G\hat p_G(x) - Gq(x)|\big],
\end{equation*}
where $G\hat p_G(x)$ is binomial with $G$ trials and success probability $q(x)$. For $G \ge 2$, a binomial with $G$ trials and success probability $v \in [1/G,\, 1 - 1/G]$ has mean absolute deviation at least $\sqrt{Gv(1-v)/2}$ \citep[Theorem~1]{berend2013sharp}. Applying this bound with $v = q(x)$ to each outcome in $\mathcal{X}_1$ and dropping the other outcomes, whose terms are nonnegative, gives the lower bound. \qed

\paragraph{Proof of part (iv).}
Let $\pi^\dagger$ minimize $F_G$ over all distributions on $\mathcal{X}$. Two inequalities bound its distance to $q$. First, $\tv(\cdot, q)$ is convex and $\E[\hat p_G] = \pi^\dagger$ when the draws come from $\pi^\dagger$, so Jensen's inequality gives $F_G(\pi^\dagger) \ge \tv(\pi^\dagger, q)$. Second, $\pi^\dagger$ is a minimizer, so $F_G(\pi^\dagger) \le F_G(q) = B_G(q)$. Chaining the two,
\begin{equation*}
\tv(\pi^\dagger,q) \le F_G(\pi^\dagger) \le B_G(q).
\end{equation*}

For the second claim, we give a target for which deleting the rare outcomes beats sampling from the target. Take $G = 64$ and a target with one head outcome of probability $1-\beta$ and $M$ tail outcomes of probability $\eta$ each,
\begin{equation*}
q = \big(1-\beta,\; \eta,\; \dots,\; \eta\big), \qquad \beta = 0.1, \quad M = 64, \quad \eta = \beta/M,
\end{equation*}
so every tail outcome has probability below $1/G$. We compare two distributions.
\begin{enumerate}[leftmargin=*,itemsep=1pt,topsep=2pt]
\item The point mass $\delta_0$ on the head outcome gives zero probability to every tail outcome. Its empirical distribution is always $\delta_0$ itself, so $F_G(\delta_0) = \tv(\delta_0, q) = \beta = 0.1$ exactly.
\item The target $q$ has $F_G(q) = B_G(q)$, and we bound $B_G(q)$ from below one outcome at a time. Each tail outcome $x$ contributes to $\E\big[\|\hat p_G - q\|_1\big]$ the amount
\begin{equation*}
\E\big[|\hat p_G(x) - \eta|\big] = 2\,\E\big[(\eta - \hat p_G(x))^+\big] \ge 2\eta(1-\eta)^G .
\end{equation*}
The equality holds because $\hat p_G(x) - \eta$ has mean zero. The inequality keeps the event that $x$ is never drawn, which has probability $(1-\eta)^G$. It is in fact an equality here, because $G\eta < 1$, so $\hat p_G(x)$ falls below $\eta$ only when $x$ is never drawn. The head outcome contributes at least $\sqrt{\beta(1-\beta)/(2G)}$ by the binomial bound of part (iii), since $1-\beta \in [1/G, 1-1/G]$. Summing over the $M$ tail outcomes with $M\eta = \beta$, adding the head, and halving,
\begin{equation*}
B_G(q) \;\ge\; \beta(1-\eta)^G + \frac{1}{2\sqrt2}\sqrt{\frac{\beta(1-\beta)}{G}} \;=\; 0.0905 + 0.0133 \;>\; 0.1 .
\end{equation*}
\end{enumerate}
Hence $F_G(q) > 0.1 = F_G(\delta_0)$, and the surrogate strictly prefers the distribution without the tail to the target itself. An exact binomial computation gives $B_G(q) = 0.1056$, so the margin is $0.0056$. The slack in the bound above comes from the head term, for which the binomial bound gives $0.0133$ against an exact $0.0151$. Outcomes rarer than $1/G$ cost the surrogate more in sampling fluctuation than their deletion costs in TV. \qed

\subsection{When the Bias Matters}

Parts (ii) and (iii) predict that the expected update of the group-scalar reward shrinks once the model's TV to the target falls to the size of the bias. By part (ii) and the definition of $B_G$,
\begin{equation*}
\E[U] = -\frac{K-1}{KG}\Big(\nabla_\theta \tv(\pi_\theta,q) + \nabla_\theta B_G(\pi_\theta)\Big).
\end{equation*}
Where the gradient of the bias is negligible, the expected update is a scaled negative gradient of the model's own TV. We do not bound $\nabla_\theta B_G$. We instead locate by simulation where this approximation breaks.

\paragraph{Setup.}
We take a uniform target $q$ on $d$ outcomes and a model parameterized by a softmax over outcomes, with logits $\theta$. The model moves along the path
\begin{equation*}
\pi = (1-t)\,\pi_0 + t\,q, \qquad t \in [0,1],
\end{equation*}
where $\pi_0$ is uniform on half of the outcomes and has TV $1/2$ to $q$. We use $K = 4$, $G \in \{8, 64, 256\}$, and $d \in \{2, 10, 100\}$, nine settings in all. At $15$ points of the path, with TV to $q$ spaced geometrically between $0.45$ and $0.001$, we estimate the expected update by Monte Carlo and subtract the Monte Carlo variance from its squared norm. We then divide its norm by the norm of $\frac{K-1}{KG}\nabla_\theta\tv(\pi_\theta,q)$. This ratio is close to one where the approximation holds. For a uniform target, the bias scale of part (iii) is $\frac{1}{\sqrt G}\sum_x\sqrt{q(x)(1-q(x))} = \sqrt{(d-1)/G}$. Scanning from the far end of the path, we read off the TV at which the ratio first falls below $1/2$, by linear interpolation in $\log\tv$ between the two neighboring points. We write this crossing as $\kappa\sqrt{(d-1)/G}$.

\paragraph{Result.}
The crossing is proportional to the bias scale. In the five settings where the ratio is at least $0.75$ at the two largest simulated distances, $\kappa$ lies between $0.334$ and $0.342$, and a log-log fit of the crossings against $\sqrt{(d-1)/G}$ has slope $1.01$. In two more settings the ratio is below $0.75$ at one of the two largest simulated distances, so the bias already matters at the start of the path. They give $\kappa = 0.340$ and $0.370$ if the crossing is read off anyway. In the remaining two settings, $(G,d) = (8,100)$ and $(64,100)$, the ratio is below $1/2$ along the whole path. Our wide-support families at $G = 64$ are in this regime. Wherever the norm of the Monte Carlo estimate exceeds five times its standard error, the direction of the expected update has cosine similarity at least $0.98$ with $-\nabla_\theta\tv(\pi_\theta,q)$, so only its size shrinks. Extending the scaling from uniform to general targets is an extrapolation based on part (iii), not a separate measurement.

\subsection{Numerical Checks}
\label{app:scalarchecks}

We check Lemma~\ref{lem:degeneracy}, the count of distinct advantages, and the bound of part (i) on sampled batches, and we check part (ii) against exact enumeration at $|\mathcal{X}| = 2$. We also compute the bias of part (iii), the example of part (iv), and the simulation above. The values of $\kappa$ and the slope quoted above are computed against $\sqrt{(d-1)/G}$, the scale of part (iii).

\section{Proofs for the Sign Witness}
\label{app:sign}

Proposition~\ref{prop:sign} states that the sign witness trains the model toward the target except on low-mass outcomes. We restate it and prove the three parts in order. We then simulate what the proof leaves open, and relate the sign witness to the property-ratio reward of \citet{huang2026alignment}.

The notation of Appendix~\ref{app:tools} remains in force. Rollout $i$ receives the sign witness $A_i = \mathrm{sign}\big(q(x_i) - \hat p_{-i}(x_i)\big)$, with $\mathrm{sign}(0) = 0$, so a lone invalid rollout again receives zero. The population sign of outcome $x$ is $\mathrm{sign}\big(q(x) - \pi_\theta(x)\big)$, and an outcome is low-mass if $0 < q(x) < 1/(G-1)$. We write $\bar a(x) = \E[A_i \mid x_i = x]$ for the expected advantage of a rollout with outcome $x$, which does not depend on $i$. In the proposition, the expected update is $\E\big[\tfrac1G\sum_i A_i\, s(x_i)\big]$ with the advantages before centering. The paragraph on centering in Appendix~\ref{app:signlowmass} shows that the conclusions of part (iii) also hold for the update that GRPO applies, $\E\big[\tfrac1G\sum_i \tilde A_i\, s(x_i)\big]$, when $G \ge 3$.

\begin{restatedbox}
\propsign*
\end{restatedbox}

\paragraph{Proof idea.}
Part (i) combines the dual form of total variation with the score identity. Part (ii) holds because, given the outcome of rollout $i$, the leave-one-out count is binomial and concentrates around its mean. Part (iii) uses the fact that the leave-one-out frequency takes values on a grid with spacing $1/(G-1)$. For a low-mass outcome, it is either $0$ or above the target, so the sign only records whether the outcome appeared among the other rollouts.

\subsection{Part (i): The Population Sign Is the Witness of Total Variation}

For distributions $\pi$ and $q$ and any function $f$ with $|f(x)| \le 1$ for all $x$,
\begin{equation*}
\sum_x f(x)\big(\pi(x) - q(x)\big) \le \sum_x \big|\pi(x) - q(x)\big| = 2\,\tv(\pi,q),
\end{equation*}
with equality if and only if $f(x) = \mathrm{sign}(\pi(x) - q(x))$ wherever $\pi(x) \ne q(x)$. This is the dual representation of total variation \citep{muller1997integral,sriperumbudur2012empirical}. Its maximizer, the witness of total variation, is $\mathrm{sign}(\pi - q)$, the negative of the population sign. This is the first claim of part (i).

For the expected update, assume $\pi_\theta(x) \neq q(x)$ for every $x$. Then each term $|\pi_\theta(x) - q(x)|$ is differentiable, with gradient $\mathrm{sign}(\pi_\theta(x) - q(x))\,\nabla_\theta\pi_\theta(x)$. The population sign of rollout $i$ is a function of $x_i$ alone, so Lemma~\ref{lem:reduction} with $N = G$ applies, and each of the $G$ terms of the update has the same expectation. Lemma~\ref{lem:score}(b) with $f = \mathrm{sign}(q - \pi_\theta)$ then gives
\begin{equation*}
\E_{x\sim\pi_\theta}\Big[\mathrm{sign}\big(q(x) - \pi_\theta(x)\big)\,s(x)\Big] = -\sum_x \mathrm{sign}\big(\pi_\theta(x) - q(x)\big)\,\nabla_\theta\pi_\theta(x) = -2\,\nabla_\theta\tv(\pi_\theta,q).
\end{equation*}
This is the second claim of part (i). \qed

\subsection{Part (ii): The Sign Computed From the Group Rarely Differs}

Condition on $x_i = x$. The leave-one-out count excludes rollout $i$, and the other $G-1$ draws are independent of $x_i$, so $(G-1)\,\hat p_{-i}(x)$ is exactly binomial with $G-1$ trials and success probability $\pi_\theta(x)$. Suppose $\delta = q(x) - \pi_\theta(x) > 0$, so the population sign is $+1$. The sign computed from the group is not $+1$ only if $\hat p_{-i}(x) \ge q(x)$, that is, only if $\hat p_{-i}(x) - \pi_\theta(x) \ge \delta$. This event includes the tie $\hat p_{-i}(x) = q(x)$, where the sign is $0$. Hoeffding's inequality for the mean of $G-1$ independent Bernoulli draws bounds its probability by $\exp(-2(G-1)\delta^2)$. The case $\delta < 0$ is symmetric. \qed

\subsection{Part (iii): Low-Mass Outcomes}
\label{app:signlowmass}

\paragraph{The expected advantage of a low-mass outcome.}
Fix a low-mass outcome $x$, so $0 < q(x) < 1/(G-1)$. The leave-one-out frequency $\hat p_{-i}(x)$ takes values in $\{0, \tfrac{1}{G-1}, \tfrac{2}{G-1}, \dots\}$, so it is never strictly between $0$ and $q(x)$ and never equal to $q(x)$. The sign is therefore $+1$ exactly when none of the other $G-1$ rollouts produced $x$, and $-1$ otherwise. Given $x_i = x$, the expected advantage is
\begin{equation}
\bar a(x) = (+1)\cdot\big(1-\pi_\theta(x)\big)^{G-1} + (-1)\cdot\Big(1 - \big(1-\pi_\theta(x)\big)^{G-1}\Big) = 2\big(1-\pi_\theta(x)\big)^{G-1} - 1 .
\label{eq:signcredit}
\end{equation}
The target probability $q(x)$ does not appear. The expected advantage is positive exactly when $\pi_\theta(x) < 1 - 2^{-1/(G-1)}$, which is $0.0109$ at $G = 64$ (close to $\ln 2/(G-1)$), whether or not $\pi_\theta(x)$ already exceeds its target.

\paragraph{Stationary points of the expected update.}
Suppose that the model can assign any probability to each outcome. We parameterize it by a softmax over outcomes with logits $\theta = (\theta_{x})_{x \in \mathcal{X}}$, so $\pi_\theta(x) \propto \exp(\theta_x)$ and
\begin{equation*}
\frac{\partial \pi_\theta(x)}{\partial\theta_{x'}} = \pi_\theta(x)\big(\mathbb{1}[x = x'] - \pi_\theta(x')\big).
\end{equation*}
Conditioning on $x_i$ and applying Lemma~\ref{lem:score}(b) with $f = \bar a$ gives
\begin{equation*}
\E[A_i\, s(x_i)] = \sum_x \bar a(x)\,\nabla_\theta\pi_\theta(x).
\end{equation*}
Substituting the softmax derivative, the component of this update along $\theta_{x'}$ is
\begin{equation*}
\sum_x \bar a(x)\,\pi_\theta(x)\big(\mathbb{1}[x = x'] - \pi_\theta(x')\big) = \pi_\theta(x')\Big(\bar a(x') - \sum_x \pi_\theta(x)\,\bar a(x)\Big).
\end{equation*}
For this paragraph we drop the positivity assumption of Appendix~\ref{app:tools} and read the right-hand side as a vector field on the closed simplex, where outcomes may have probability zero. There $\bar a(x')$ is given by its binomial formula, such as Equation~\ref{eq:signcredit} for a low-mass $x'$, and the component along $\theta_{x'}$ vanishes whenever $\pi(x') = 0$. We call a distribution $\pi$ a stationary point when every component vanishes, that is, when $\bar a$ takes a common value $\lambda$ on the support of $\pi$. Outcomes outside the support, which are the outcomes that the stationary point sets to zero, impose no condition. The condition does not depend on the parameterization. In the interior of the simplex, $\sum_x \bar a(x)\,\nabla_\theta\pi_\theta(x) = 0$ holds for any parameterization that can move $\pi$ in every direction of the simplex exactly when $\bar a$ is constant.

At a stationary point, every low-mass outcome $x$ in the support satisfies $2(1-\pi(x))^{G-1} - 1 = \lambda$ by Equation~\ref{eq:signcredit}. The map $v \mapsto 2(1-v)^{G-1} - 1$ is strictly decreasing on $[0,1]$, so all these outcomes have the same probability, $1 - \big((1+\lambda)/2\big)^{1/(G-1)}$. Their total probability is this value times their number. It depends on how many low-mass outcomes the support contains and on the common level $\lambda$, but not on their target probabilities.

\paragraph{When the target is not stationary.}
At $\pi = q$, the support is the support of $q$. Suppose two low-mass outcomes $x$ and $x'$ have $q(x) \ne q(x')$. Equation~\ref{eq:signcredit} gives them the expected advantages $2(1-q(x))^{G-1} - 1$ and $2(1-q(x'))^{G-1} - 1$, which differ because the map above is strictly decreasing. The expected advantage is therefore not constant on the support, and $q$ is not a stationary point. A Zipf target gives a different probability to each outcome, so it satisfies this condition as soon as it has two low-mass outcomes.

The requirement of two low-mass outcomes with different target probabilities cannot be dropped. Consider the binary target with outcomes $a$ and $b$ and
\begin{equation*}
q(a) = 1-2^{-1/(G-1)}, \qquad q(b) = 2^{-1/(G-1)} .
\end{equation*}
For every $G \ge 2$ this target is stationary, because both outcomes have expected advantage $0$. Outcome $a$ is low-mass, and Equation~\ref{eq:signcredit} gives it the expected advantage $2\cdot 2^{-1} - 1 = 0$. For $G \ge 3$, outcome $b$ is not low-mass, so we compute its expected advantage directly. Since $q(b) > (G-2)/(G-1)$, the only value of the leave-one-out frequency at or above $q(b)$ is $1$. The sign of $b$ is therefore $-1$ exactly when all $G-1$ other rollouts produce $b$, which has probability $q(b)^{G-1} = 1/2$, and $+1$ otherwise. Its expected advantage is $\tfrac12 - \tfrac12 = 0$. For $G = 2$, both outcomes are low-mass and have expected advantage $0$ by Equation~\ref{eq:signcredit}.

\paragraph{Centering does not change these conclusions.}
GRPO trains on the centered advantage $\tilde A_i = A_i - \bar A$. For $G \ge 3$, the two conclusions about stationary points also hold for the expected centered update, with $G-2$ in place of $G-1$ in the exponent. To show this, we compute the expected centered advantage of a low-mass outcome in three steps. It turns out to be a strictly decreasing function of $\pi_\theta(x)$ plus a constant, as in Equation~\ref{eq:signcredit}.

\textbf{Step 1: separate the rollout's own advantage.} By the same conditioning argument as above, the expected centered update is $\sum_x \bar a_{\mathrm{c}}(x)\,\nabla_\theta\pi_\theta(x)$. Since $\bar A$ contains $A_i/G$, the expected centered advantage is
\begin{equation*}
\bar a_{\mathrm{c}}(x) = \E[\tilde A_i \mid x_i = x] = \frac{G-1}{G}\,\bar a(x) - \frac1G\sum_{j \ne i}\E[A_j \mid x_i = x].
\end{equation*}

\textbf{Step 2: the advantage of another rollout.} Fix $j \ne i$ and condition also on $x_j = x'$. Let $c$ be the number of the $G-2$ rollouts other than $i$ and $j$ whose outcome is $x'$, a binomial count with $G-2$ trials and success probability $\pi_\theta(x')$. Define
\begin{equation*}
h(x') = \E\Big[\mathrm{sign}\Big(q(x') - \frac{c}{G-1}\Big)\Big], \qquad h^+(x') = \E\Big[\mathrm{sign}\Big(q(x') - \frac{c+1}{G-1}\Big)\Big].
\end{equation*}
If $x' \ne x$, rollout $i$ does not count toward $\hat p_{-j}(x')$, so $\E[A_j \mid x_i = x, x_j = x'] = h(x')$. If $x' = x$, rollout $i$ adds one to the count, and the conditional expectation is $h^+(x)$. Averaging over $x'$,
\begin{equation*}
\E[A_j \mid x_i = x] = \sum_{x'} \pi_\theta(x')\,h(x') + \pi_\theta(x)\big(h^+(x) - h(x)\big).
\end{equation*}
The first term does not depend on $x$. The expression is the same for every $j \ne i$, so the sum in Step 1 is $G-1$ times it.

\textbf{Step 3: evaluate at a low-mass outcome.} For a low-mass $x$, the grid argument gives $h(x) = 2(1-\pi_\theta(x))^{G-2} - 1$ and $h^+(x) = -1$, because a count of at least one puts the leave-one-out frequency at $1/(G-1) > q(x)$ or above. Write $v = \pi_\theta(x)$. Substituting Step 2 and Equation~\ref{eq:signcredit} into Step 1,
\begin{align*}
\bar a_{\mathrm{c}}(x) &= \frac{G-1}{G}\Big(2(1-v)^{G-1} - 1 + 2v(1-v)^{G-2}\Big) + C \\
&= \frac{G-1}{G}\Big(2(1-v)^{G-2}\big[(1-v) + v\big] - 1\Big) + C
= \frac{G-1}{G}\Big(2(1-v)^{G-2} - 1\Big) + C,
\end{align*}
where $C = -\frac{G-1}{G}\sum_{x'} \pi_\theta(x')\,h(x')$ does not depend on $x$. The first line uses $-v\big(h^+(x) - h(x)\big) = 2v(1-v)^{G-2}$.

For $G \ge 3$, $\bar a_{\mathrm{c}}$ is therefore a constant plus a strictly decreasing function of $\pi_\theta(x)$ alone. At a stationary point of the centered update, all low-mass outcomes in the support again share one probability, and $\pi = q$ is again not stationary when two low-mass outcomes have different target probabilities. The constant $C$ moves the point where $\bar a_{\mathrm{c}}$ changes sign away from $1 - 2^{-1/(G-1)}$, and the binary example above is stationary only for the uncentered update.

\subsection{Simulating the Sign Witness on Zipf Targets}

The proof shows that the sign witness cannot hold low-mass outcomes with different target probabilities at their targets. It does not say whether their total probability ends too high or too low. We therefore simulate the replicator flow of the expected update,
\begin{equation*}
\frac{d}{dt}\log\pi_t(x) = \bar a(x) - \sum_{x'} \pi_t(x')\,\bar a(x'),
\end{equation*}
which has the same stationary points. Here the subscript $t$ is time, and $\bar a$ is evaluated at $\pi_t$. We use $G = 64$ and three Zipf targets, with $200$ outcomes and exponent $1.2$, $1{,}000$ outcomes and exponent $1.5$, and $50$ outcomes and exponent $2.0$. The flow starts from $\pi_0(x) \propto x^{-2.5}$ on the outcomes $x = 1, 2, \dots$ of the target. The sign flow raises the total probability of the low-mass outcomes to $1.5$ to $2.4$ times their target probability and stops at TV $0.13$ to $0.47$ from the target. The same flow with the witness advantage reaches the target, with TV below $10^{-4}$.

\subsection{Connection to Property-Ratio Alignment}
\label{app:propertyratio}

\citet{huang2026alignment} train with a per-group $0/1$ reward that pushes the fraction of rollouts with a binary property toward a target ratio $r_t$ (their notation). Let $x_i \in \{0,1\}$ be the property value of rollout $i$, let $\hat r$ be the group's empirical ratio, and let $\epsilon \ge 0$ be a tolerance. We assume that every rollout is valid, since they give invalid rollouts reward $0$ and leave them out of $\hat r$. Their reward is
\begin{equation*}
R_i = \begin{cases} x_i & \text{if } \hat r < r_t - \epsilon, \\ 1 - x_i & \text{if } \hat r > r_t + \epsilon, \\ 0 & \text{otherwise,} \end{cases}
\end{equation*}
and training stops once the ratio lies inside the band. Set $q = \mathrm{Bernoulli}(r_t)$ on $\{0,1\}$. The sign witness computed from the full group is
\begin{equation*}
\sigma_i = \mathrm{sign}\big(q(x_i) - \hat p(x_i)\big) = (2x_i - 1)\,\mathrm{sign}(r_t - \hat r),
\end{equation*}
since $q(1) - \hat p(1) = r_t - \hat r$ and $q(0) - \hat p(0) = \hat r - r_t$. On a group whose ratio lies outside the band, checking the two cases shows $R_i = (1 + \sigma_i)/2$. Constants cancel under centering, so $R_i - \bar R = (\sigma_i - \bar\sigma)/2$. If the group also contains both property values, both standard deviations are positive and the scale cancels too, so their standardized GRPO advantage equals the standardized sign witness. Inside the band, their reward is zero for every rollout, whereas $\sigma_i$ is nonzero unless $\hat r = r_t$. \citet{huang2026alignment} also give a multi-class variant. It rewards rollouts from under-represented attributes with $1$ and those from over-represented attributes with $0$, and removes the rollouts of attributes whose ratio lies inside the band. On a finite alphabet, this is the full-group sign witness shifted and scaled as above, restricted to outcomes outside the band. Their binary targets are $r_t = 0.5$ and $r_t = 0.05$, with rollout sizes of $500$ and $800$, so neither outcome is low-mass and the regime of part (iii) does not arise in their binary setting.

\section{Experimental Setup}
\label{app:setup}

This appendix gives the details behind Section~\ref{sec:impl}: the prompts, the parser, the choice of targets, the training settings, and the evaluation and statistics.

\subsection{Prompts}

Every draw prompt for a synthetic target uses the system message below. The training targets, the unseen-parameter targets, and the hypergeometric targets of Table~\ref{tab:paired} use the original format. The user message in this format states the family and its parameters and asks for one outcome. For the geometric, Poisson, and Zipf families it also gives the PMF formula. The biased-coin prompt gives both probabilities. The binomial prompt gives the parameters and the range of outcomes, and the hypergeometric prompt gives only the parameters. Two examples, for the target of Figure~\ref{fig:phenomenon}a and for a geometric target:
\begin{promptbox}
\textbf{System.} You simulate random draws from probability distributions. When asked for a draw, you output exactly one outcome and nothing else.

\textbf{User.} A biased coin lands on Heads with probability 0.005 and on Tails with probability 0.995. Flip the coin once and report the single outcome. Respond with exactly one word {-}{-} either `Heads' or `Tails' {-}{-} and nothing else.

\textbf{User.} A geometric distribution has success probability p = 0.551. Its outcomes are the positive integers 1, 2, 3, \ldots\ (the number of independent trials up to and including the first success), with probability P(k) = (1-p)\^{}(k-1) * p. Draw a single sample from this distribution and report the single integer outcome. Respond with only the integer and nothing else.
\end{promptbox}
The unseen-family prompts use the evaluation format, a second template with the same system message. The retrain on the evaluation format (Section~\ref{sec:synthetic}) uses this template for its training targets as well. The template states the family and its parameters, gives a PMF formula for the maximum of dice and the logarithmic series and a description in words for the other families, and adds a sentence that names the valid outcomes. An example for the maximum of dice:
\begin{promptbox}
\textbf{User.} A maximum of dice distribution has the maximum of k = 3 independent dice, each with m = 4 sides. The probability that the maximum equals x is (x/m)\^{}k - ((x-1)/m)\^{}k. The valid outcomes are the integers from 1 to 4 inclusive. Draw one random sample from this distribution. Respond with only the outcome and nothing else.
\end{promptbox}

\paragraph{Stating prompts.}
To test whether the model knows the target, a separate prompt asks it to name the distribution and state its parameters, under the system message ``You are a precise probability assistant. Answer concisely in the exact format requested.'' There is one template per family. The geometric template reads as follows.
\begin{promptbox}
\textbf{User.} Consider a random process of repeated independent trials, each succeeding with probability 0.551, where you count the number of trials up to and including the first success. Name the probability distribution this follows and state its parameter value. Answer in exactly this format: `\textless{}distribution name\textgreater{}, p=\textless{}number\textgreater{}'.
\end{promptbox}
The coin template asks instead for the probability of each outcome. An answer is scored correct when the family keyword appears (for every family except the coin) and every parameter matches: probabilities within $0.0015$, integers exactly, and the Poisson rate within $0.5\%$ or $0.01$, whichever is larger.

\subsection{Parser and Invalid Outputs}

The parser is strict. It strips surrounding whitespace, one layer of matching quotes or backticks, and one trailing period, and the rest of the response must be exactly one outcome. For the biased coin, the outcome must be Heads or Tails, in any capitalization. For the integer families, it must be an integer inside the support of the target. Everything else maps to $\bot$. The common cases are extra words (``The outcome is 3''), more than one outcome, a non-integer, and an integer outside the support, such as $0$ for a geometric target or $9$ for a binomial target with $8$ trials. Training and evaluation use the same parser. In training, $\bot$ is an outcome with target probability $0$, so an invalid rollout receives the witness advantage $-2\hat p_{-i}(\bot)$ before centering, which is negative whenever another rollout of the group is also invalid. In evaluation, invalid responses are left out of the empirical distribution and their rate is reported separately.

\subsection{Targets}

For each training family, we take the candidate parameter settings that the Spectrum Suite code generates \citep{sorensen2026spectrum}. We keep a candidate if a perfect sampler's expected TV at $5{,}000$ draws is at most $0.10$, estimated from $1{,}000$ Monte Carlo replicates. This screen removes targets with more outcomes than a few thousand draws can resolve. From the kept candidates, we select $20$ settings spaced by quantiles of the primary parameter, such as $p$ for the geometric family. We sort them by this parameter and number them from $0$. Ranks $2$, $7$, $12$, and $17$ form the unseen-parameter set, and the other $16$ are training targets. Each target is also assigned the smallest $n \in \{500, 2{,}000, 5{,}000\}$ at which the perfect sampler's expected TV is at most $0.10$. This $n$ is the number of evaluation draws for unseen-parameter targets (Appendix~\ref{app:stats}). The five unseen families are screened the same way, with $20$ targets each. Hypergeometric targets come from the Spectrum Suite code, and the other four families from our own samplers. In the rest of the appendix, we write occupancy for the number of empty boxes, triangular for the discrete triangular, and log-series for the logarithmic series.

\subsection{Training}
\label{app:training}

Table~\ref{tab:hparams} lists the settings shared by all GRPO runs. We follow \citet{liu2025understanding} in two choices. First, the loss normalizes by a constant length instead of each response's length. Second, the advantages are centered by the group mean but not divided by the group's standard deviation of rewards. The original GRPO \citep{shao2024deepseekmath} divides by this standard deviation. We do not, because the size of the witness advantage measures how far the model is from the target (Section~\ref{sec:witness}), and dividing would remove this information and change the expected update derived in Appendix~\ref{app:witness}.

\begin{table}[h]
\centering
\caption{Settings shared by all GRPO runs. Exceptions are listed in the text.}
\label{tab:hparams}
\small
\renewcommand{\arraystretch}{1.08}
\begin{tabular}{@{}lp{0.6\linewidth}@{}}
\toprule
setting & value \\
\hmidrule
model & Qwen2.5-1.5B-Instruct \citep{yang2024qwen25}, full fine-tuning, bfloat16 mixed precision \\
implementation & GRPO in TRL \citep{vonwerra2020trl}, generation with vLLM \citep{kwon2023vllm} on the same GPU \\
group size $G$ & $64$ \\
prompts per step & $4$ ($256$ rollouts) \\
steps & $600$, one gradient step per generation batch (on-policy) \\
optimizer & AdamW, constant learning rate $2\times10^{-6}$, $\beta_1=0.9$, $\beta_2=0.99$, no weight decay \\
gradient clipping & $1.0$ \\
PPO clip range $\epsilon$ & $0.2$ \\
KL weight toward the initial model & $0.02$ \\
loss normalization & constant length \citep{liu2025understanding} \\
advantage & reward minus group mean, no division by the standard deviation \\
decoding in training & temperature $1$, at most $24$ response tokens \\
hardware and time & one 48\,GB GPU; about one GPU-hour per main run on the synthetic distributions, and $1.5$ to $1.7$ GPU-hours per run of the group-size sweep \\
\bottomrule
\end{tabular}
\end{table}

\paragraph{Exceptions.}
The group-scalar reward uses one prompt of $256$ rollouts per step, split into four subgroups of $64$, and is otherwise identical. The group-size sweep keeps $256$ rollouts per step. The witness runs use $256/G$ prompts per step, from $32$ at $G = 8$ to $1$ at $G = 256$, and the group-scalar runs use one prompt split into $256/G$ subgroups. The entropy-bonus baseline adds TRL's entropy bonus with coefficient $0.01$ to a reward for valid output. Its mean valid rate during training was $0.95$. The opinion-distribution runs use $1{,}600$ training prompts ($800$ urn and $800$ GlobalOpinionQA) and take about $2.5$ GPU-hours. The sibling-discovery runs use $400$ steps and take $1.3$ to $2.4$ GPU-hours, and their evaluation allows up to $32$ response tokens.

\paragraph{Cross-entropy baseline.}
The supervised baseline minimizes, for each training prompt, $-\sum_x q(x)\log P_\theta(y_x)$, where $y_x$ is the exact response string for outcome $x$ followed by the end-of-turn token. This loss replaces the outcome probability $\pi_\theta(x)$ in the loss of Section~\ref{sec:impl} by the probability of one canonical response, which is a lower bound on $\pi_\theta(x)$. The sum runs over a finite set of outcomes of each training target, on which $q$ is renormalized. For the coin it has both outcomes. For the other families except Zipf it has the outcomes between the $10^{-4}$ and $1-10^{-4}$ quantiles of the target, and for Zipf targets the most likely outcomes that cover $99\%$ of the probability, at most $2{,}000$. We fully fine-tune the same model with AdamW (constant learning rate $10^{-5}$, $\beta_1 = 0.9$, $\beta_2 = 0.999$, no weight decay, gradient clipping $1.0$) on minibatches of $8$ prompts. A time cap of $36$ minutes stopped training after $11$ epochs over the $80$ training targets ($110$ steps, about $0.6$ GPU-hours), and we evaluate the final checkpoint. We did not tune this recipe to preserve capability. Unlike our GRPO runs, it has no KL term toward the initial model, and its learning rate is five times the GRPO learning rate. \citet{zhang2024forcing} instead tune low-rank adapters for at most $50$ steps, stopping early on tasks whose target cannot be enumerated, and report little change on MT-Bench. Our full fine-tune costs $10$ MMLU points and $23$ IFEval points (Section~\ref{sec:synthetic}). This cost therefore reflects our recipe as well as the objective.

\subsection{Evaluation and Statistics}
\label{app:stats}

\paragraph{Sampling error.}
Training targets and unseen families are evaluated with $n = 500$ draws per target. Unseen-parameter targets and the targets of Figure~\ref{fig:phenomenon}b use the $n$ assigned above. Excess TV subtracts the expected TV of a perfect sampler at the same $n$. We estimate this quantity once for each target from $2{,}000$ Monte Carlo replicates. A reduction in excess TV compares medians over the targets of a set, $1 - \mathrm{median}(\text{excess TV}_{\text{trained}})/\mathrm{median}(\text{excess TV}_{\text{untrained}})$.

\paragraph{Capability.}
MMLU uses the full test set ($14{,}042$ questions) with greedy decoding and counts an answer as correct when the first letter from A to D in the response is the right one. IFEval uses its $541$ prompts and reports instruction-level strict accuracy. The noise band of each benchmark is the $95\%$ bootstrap interval of the untrained model's score, from $2{,}000$ replicates. Its half-width ($0.008$ for MMLU and $0.033$ for IFEval) was fixed before training.

\paragraph{Intervals.}
Unless stated otherwise, an interval is a $95\%$ paired bootstrap interval over the targets of a set, from $2{,}000$ replicates. Rows that pool seeds resample targets within each seed. Capability intervals resample benchmark items, clustered by prompt for IFEval. On GlobalOpinionQA, the intervals of the reduction resample question stems.

\section{Additional Results on Synthetic Distributions}
\label{app:more}

This appendix adds detail to Section~\ref{sec:synthetic}. We give per-family results, the full comparison of the three rewards, the training-free baselines, a split of the remaining error, a test of whether the model relies on the listed outcomes, and a study of how transfer depends on the variety of training targets.

\subsection{Results per Family}

Table~\ref{tab:synthfull} reports all three synthetic evaluation sets for the untrained model, supervised cross-entropy, and the witness advantage. Table~\ref{tab:families} gives the median excess TV per family, for the three rewards on the training families and for the two witness checkpoints on the unseen families. The witness advantage lowers excess TV on every training family, and its largest remaining errors are on Poisson and binomial targets (Appendix~\ref{app:decomposition}). On the unseen families, the checkpoint trained on the original prompts improves four families and leaves occupancy about the same ($0.615$ to $0.623$). The retrain on prompts in the evaluation format improves all five. Supervised cross-entropy, trained on the original prompts, removes $31\%$ of the excess TV on the unseen families, against $36\%$ for the witness advantage trained on the same prompts.

\begin{table}[h]
\centering
\caption{Sampling error on the three synthetic evaluation sets and general capabilities. Supervised cross-entropy fits the training targets and unseen parameters more closely than the witness advantage but loses $10$ MMLU and $23$ IFEval points. Sampling error is the median excess TV on the 64 training and 16 unseen-parameter targets outside the Zipf family and on the 100 unseen-family targets. Capability is accuracy, with 95\% bootstrap half-widths of $0.008$ (MMLU) and $0.033$ (IFEval).}
\label{tab:synthfull}
\small
\setlength{\tabcolsep}{7pt}
\setlength{\aboverulesep}{0pt}
\setlength{\belowrulesep}{0pt}
\renewcommand{\arraystretch}{1.2}
\begin{tabular}{@{}p{2.25cm}lcc>{\columncolor{accent!7}[\tabcolsep][0pt]}c@{}}
\toprule
 & & untrained & cross-entropy & witness (ours) \\
\hmidrule
\multirow{3}{=}{sampling error} & training targets & 0.477 & 0.053 & 0.101 \\
 & unseen parameters & 0.483 & 0.094 & 0.161 \\
 & unseen families & 0.589 & 0.406 & 0.376 \\
\hmidrule
\multirow{2}{=}{capability} & MMLU & 0.584 & 0.482 & 0.583 \\
 & IFEval & 0.512 & 0.281 & 0.528 \\
\bottomrule
\end{tabular}
\end{table}

\begin{table}[h]
\centering
\caption{Median excess TV per family. Top: training families after 600 steps, with the lowest value per family in bold. Bottom: unseen families and all 100 unseen-family targets pooled, for the untrained model and for the witness advantage trained on the original prompts or on prompts in the evaluation format. All values use $n = 500$ draws per target. Figure~\ref{fig:phenomenon}b uses the $n$ assigned to each target, so its untrained values differ from those in this table.}
\label{tab:families}
\small
\renewcommand{\arraystretch}{1.08}
\begin{tabular*}{0.86\linewidth}{@{\extracolsep{\fill}}lcccc@{}}
\toprule
training family & untrained & witness (ours) & group-scalar & sign witness \\
\hmidrule
biased coin & 0.361 & 0.034 & 0.440 & \textbf{0.021} \\
geometric & 0.372 & 0.050 & 0.309 & \textbf{0.028} \\
binomial & 0.516 & \textbf{0.176} & 0.511 & 0.255 \\
Poisson & 0.691 & 0.421 & 0.696 & \textbf{0.263} \\
Zipf & 0.563 & \textbf{0.024} & 0.447 & 0.523 \\
\bottomrule
\end{tabular*}

\vspace{9pt}
\begin{tabular*}{0.86\linewidth}{@{\extracolsep{\fill}}lccc@{}}
\toprule
 & & \multicolumn{2}{c}{witness advantage (ours)} \\
\hcmidrule{\cmidrule(lr){3-4}}
unseen family & untrained & original prompts & evaluation-format prompts \\
\hmidrule
log-series & 0.606 & 0.139 & 0.197 \\
triangular & 0.299 & 0.221 & 0.184 \\
maximum of dice & 0.662 & 0.436 & 0.522 \\
hypergeometric & 0.654 & 0.564 & 0.545 \\
occupancy & 0.615 & 0.623 & 0.547 \\
\hmidrule
pooled (100 targets) & 0.589 & 0.376 & 0.447 \\
\bottomrule
\end{tabular*}
\end{table}

\subsection{Comparing the Three Rewards}

Table~\ref{tab:paired} gives the paired comparisons on unseen targets behind Section~\ref{sec:synthetic}. The witness advantage beats the group-scalar reward on unseen parameters in each of three seeds. The sign witness is indistinguishable from the witness advantage on $36$ unseen targets outside the Zipf family. These are the $16$ unseen-parameter targets of the other four training families and the $20$ hypergeometric targets, evaluated with prompts in the original format. On the Zipf training targets the two differ (Figure~\ref{fig:training}b).

\begin{table}[h]
\centering
\caption{Paired per-target differences in TV on unseen targets outside the Zipf family. The top block uses the $16$ unseen-parameter targets, and the bottom block adds $20$ hypergeometric targets. Each difference is the TV of the second method minus the TV of the first, so a positive value favors the first. The superscript gives the larger side of the 95\% bootstrap interval.}
\label{tab:paired}
\small
\setlength{\tabcolsep}{8pt}
\renewcommand{\arraystretch}{1.1}
\begin{tabular}{@{}p{3.9cm}lc@{}}
\toprule
 & & median difference \\
\hmidrule
\multirow{4}{=}{witness advantage vs.\ group-scalar reward\newline\tsub{\footnotesize 16 unseen-parameter targets, no Zipf}} & seed 0 & \ci{0.248}{0.131} \\
 & seed 1 & \ci{0.153}{0.137} \\
 & seed 2 & \ci{0.223}{0.129} \\
 & pooled & \ci{0.211}{0.073} \\
\hmidrule
\multirow{2}{=}{sign witness vs.\newline\tsub{\footnotesize 36 unseen targets}} & group-scalar reward & \ci{0.075}{0.123} \\
 & witness advantage & \ci{0.014}{0.023} \\
\bottomrule
\end{tabular}
\end{table}

On the unseen families, the witness advantage and the group-scalar reward are closer. Table~\ref{tab:groupsize} varies the group size at a fixed $256$ rollouts per step, and the intervals of the two rewards overlap at every group size. Two further sets of witness runs have no group-scalar counterpart. At $G = 256$, the witness advantage removes $31.1\%$ of the excess TV (interval $[22.1, 40.6]$). A second seed at $G = 8$, $16$, $32$, and $128$ removes $38.2\%$, $40.3\%$, $36.9\%$, and $32.0\%$.

\begin{table}[h]
\centering
\caption{Group-size sweep with 256 rollouts per step, seed 0. Each entry is the reduction in excess TV on the 100 unseen-family targets, in percent, and the superscript gives the larger side of the 95\% paired bootstrap interval. All runs were trained on the original prompts and evaluated on the unseen-family prompts, so they compare with the $36\%$ of the main run, not with the $24\%$ of the retrain on the evaluation format. The $G=64$ witness run is separate from the main run.}
\label{tab:groupsize}
\small
\setlength{\tabcolsep}{12pt}
\renewcommand{\arraystretch}{1.1}
\begin{tabular}{@{}ccc@{}}
\toprule
group size $G$ & witness advantage (ours) & group-scalar reward \\
\hmidrule
8 & \ci{39.4}{12.9} & \ci{30.8}{10.4} \\
16 & \ci{36.2}{13.8} & \ci{30.3}{10.7} \\
32 & \ci{35.8}{10.8} & \ci{28.6}{8.7} \\
64 & \ci{34.4}{11.9} & \ci{29.8}{12.9} \\
128 & \ci{35.8}{9.5} & \ci{24.6}{12.7} \\
\bottomrule
\end{tabular}
\end{table}

\paragraph{A reward for diversity alone.}
Rewarding diversity without reference to the target makes sampling worse than no training. Outside the Zipf family, the entropy-bonus run (Appendix~\ref{app:training}) reaches a median excess TV of $0.663$ on the training targets, against $0.477$ for the untrained model, and $0.730$ on the unseen parameters, against $0.483$.

\subsection{Training-Free Baselines}
\label{app:alternatives}

Table~\ref{tab:trainingfree} compares the training-free baselines of Section~\ref{sec:synthetic} on the unseen families. None of them removes more than $19\%$ of the excess TV, against $36\%$ for training with the witness advantage. All settings use the untrained Qwen2.5-1.5B-Instruct and $500$ draws per target unless stated otherwise. Settings that apply a method on top of the witness-trained model use a reproduction of the main run, trained with the same recipe. It reaches a pooled excess TV of $0.380$, against $0.376$ for the main run, and their paired per-target difference is indistinguishable from zero (interval $[-0.002, 0.002]$).

\begin{table}[h]
\centering
\caption{Training-free methods on the 100 unseen-family targets, applied to the untrained Qwen2.5-1.5B-Instruct, whose median excess TV is $0.589$. Reduction is relative to the untrained model. Both temperatures are chosen with knowledge of the targets.}
\label{tab:trainingfree}
\small
\setlength{\tabcolsep}{12pt}
\renewcommand{\arraystretch}{1.1}
\begin{tabular}{@{}lcc@{}}
\toprule
method & excess TV & reduction \\
\hmidrule
renormalization & 0.600 & $-2\%$ \\
seed conditioning & 0.546 & $7\%$ \\
verbalized sampling (best case) & 0.518 & $12\%$ \\
global temperature ($2.76$) & 0.514 & $13\%$ \\
per-target oracle temperature & 0.478 & $19\%$ \\
\hmidrule
witness advantage (ours) & \textbf{0.376} & $\mathbf{36\%}$ \\
\bottomrule
\end{tabular}
\end{table}

\paragraph{Verbalized sampling.}
Verbalized sampling \citep{zhang2025verbalized} asks the model to list outcomes with probabilities and samples from the parsed list. Our prompt asks for five outcomes with probabilities. We parse the list with the same parser, renormalize it, and draw $500$ outcomes from it. With one list per target, most lists do not parse. The published prompt, a variant that also names the support, and the published XML form give pooled excess TV between $0.945$ and $0.959$, with a parse rate of $0.05$ to $0.45$. With the published prompt, $46$ of $100$ targets yield lists whose probabilities are all zero, and $14$ more list only outcomes outside the support.

Table~\ref{tab:trainingfree} reports the best case, which pools ten verbalizations at temperature $1$ with the support-naming prompt ($0.518$). At least one of the ten lists parses on every target. Even the lists that parse are wrong. Their median TV to the target is $0.71$, and $0.78$ for greedy lists. Applied on top of the witness-trained model, verbalized sampling is equally poor ($0.957$ with the support-naming prompt). \citet{zhang2025verbalized} evaluate verbalized sampling on frontier API models and on open models with 70B parameters or more, and report that its gains grow with model scale, so our results describe it only at 1.5B.

\paragraph{Seed conditioning.}
Seed conditioning \citep{nagarajan2025dice} trains and tests with a random string at the start of the prompt. We add such a string at inference only, here with $16$ characters. With greedy decoding, three seed formats, one of them with $64$ characters, give an excess TV of $0.834$ to $0.847$, against $0.904$ for greedy decoding without a seed, and a median of two to four distinct outcomes per target. At temperature $1$, the seeded prompt reaches $0.546$ and the same prompt with a constant seed reaches $0.542$ (paired difference $0.0001$, interval $[-0.003, 0.005]$), so the randomness of the seed contributes nothing. \citet{nagarajan2025dice} add seeds during both training and sampling, and we test only the sampling part.

\paragraph{Renormalization.}
Renormalization sets the probability of every invalid outcome to zero and rescales the rest to sum to one. Applied to the model's sequence-level probabilities over the valid outcomes, it gives $0.600$ (paired difference from plain sampling $+0.007$, interval $[0.002, 0.013]$). Taking the most likely valid outcome gives $0.903$. Editing the logits of a single token is possible on only $29$ of the $100$ targets, where it reaches $0.725$. On top of the witness-trained model, renormalization changes little ($0.373$ against $0.380$).

\paragraph{Temperature.}
A temperature chosen for each target to minimize the exact TV to $q$ gives $0.478$. The single best temperature for all targets, $2.76$, gives $0.514$. Both are tuned with knowledge of the targets, so they bound what temperature scaling can do. Replacing the model's distribution by $q$ itself gives $0.001$, the excess TV of a perfect sampler, which is zero up to Monte Carlo error. On the $64$ training targets outside the Zipf family and the $36$ unseen targets of Table~\ref{tab:paired}, temperatures $1.3$ and $1.5$ reduce the excess TV by at most $9.5\%$ on the training targets and $6.5\%$ on the $36$ unseen targets, and min-$p$ sampling and an explicit instruction to sample at random increase it.

\paragraph{Larger models.}
Scaling the untrained model does not help. On the unseen families, Qwen2.5-3B-Instruct, Qwen2.5-7B-Instruct \citep{yang2024qwen25}, Llama-3.2-3B-Instruct, and Llama-3.1-8B-Instruct \citep{grattafiori2024llama3} all have a higher median excess TV than the 1.5B model, $0.63$ to $0.72$ against $0.59$ (Figure~\ref{fig:phenomenon}c).

\subsection{Where the Remaining Error Is}
\label{app:decomposition}

Figure~\ref{fig:mechanism} reports two diagnostics. Panel (a) checks part (i) of Proposition~\ref{prop:scalar} in training. Panel (b) splits the remaining error of the witness advantage on its two hardest families.

\begin{figure}[h]
\centering
\includegraphics[width=0.85\linewidth]{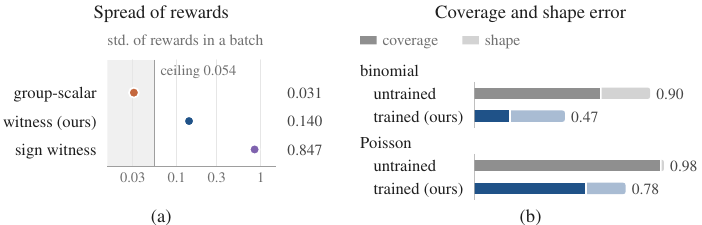}
\caption{(a)~Standard deviation of the rewards in a batch, averaged over training, on a log scale. For the group-scalar reward, it equals the standard deviation of the advantages, and it lies below the ceiling of Proposition~\ref{prop:scalar}(i) (vertical line). (b)~TV on binomial and Poisson training targets, split into a coverage term (dark) and a shape term (light), for the untrained model and after training with the witness advantage. Each bar stacks the medians of the two terms over four targets, so its total can differ from the median TV. Training mainly reduces the coverage term.}
\label{fig:mechanism}
\end{figure}

\paragraph{Spread of the rewards.}
For each run we log the standard deviation of the $256$ rewards in a batch and average it over the $600$ steps. It is $0.031$ for the group-scalar reward, $0.140$ for the witness advantage, and $0.847$ for the sign witness. For the group-scalar reward, all rewards in a batch come from one prompt and differ from the advantages by a constant, so $0.031$ is also the spread of its advantages. It lies below the ceiling $\sqrt{3/4}/(2\sqrt{64}) = 0.054$ of Proposition~\ref{prop:scalar}(i). For the other two rewards, a batch contains four prompts and GRPO subtracts a separate mean for each, so their logged spread also includes differences between prompts. The group-scalar reward gives at most four distinct advantages per batch, whereas the witness advantage can give $256$.

\paragraph{Coverage and shape.}
On binomial and Poisson targets, the untrained model puts much of its probability on outcomes that the target almost never produces. To separate this error from errors in the relative probabilities, we split the outcomes of each target into its bulk, the interval $[F^{-1}(10^{-4}), F^{-1}(1-10^{-4})]$, where $F$ is the target's cumulative distribution function, and the rest. The TV then splits exactly into two terms. The coverage term is half the absolute error on produced outcomes outside the bulk plus half the target probability of the outcomes the model never produced. The shape term is half the absolute error on produced outcomes inside the bulk. We compute both from $500$ draws on the four binomial and four Poisson training targets with the highest excess TV after training.

Before training, most of the error is coverage. On binomial targets the untrained model's median TV is $0.87$, and the median share of coverage is $71\%$. On Poisson targets the median TV is $0.98$, and the median share of coverage is $98\%$, with a sample mean a tenth of the target mean. Training with the witness advantage lowers the coverage term from $0.65$ to $0.18$ on binomial targets and from $0.95$ to $0.57$ on Poisson targets, and raises the ratio of the Poisson sample mean to the target mean from $0.10$ to $0.51$. After training, the shape terms are $0.29$ and $0.21$. The remaining binomial error is therefore mostly shape, while most of the remaining Poisson error is still coverage. On both families the draws are much more spread out than the target. The median ratio of the standard deviation of the draws to that of the target is between $7$ and $15$, before and after training.

\subsection{Copying the Listed Outcomes}
\label{app:attribution}

Every unseen-family prompt names the valid outcomes. A model could lower its excess TV by copying this list without reading the rest of the prompt. We first explain why the prompts list the outcomes. We then remove the list at evaluation and find that most of the improvement remains.

\paragraph{Why the prompts list the outcomes.}
Before we fixed the unseen-family prompts, we evaluated the witness checkpoint on the $20$ hypergeometric targets in the original format, whose prompts do not name the valid outcomes. Training moved excess TV only from $0.625$ to $0.614$, and the fraction of valid outputs fell from $0.64$ to $0.40$. When few items are drawn, the support is tiny. When one item is drawn, it is $\{0,1\}$, and the model emits counts far above it. A model cannot know the support of a family it has never seen, so the unseen-family prompts name it.

\paragraph{Removing the list.}
We evaluate the untrained and the trained model with and without the sentence that names the valid outcomes, on all $100$ unseen-family targets. The trained model is the reproduction of Appendix~\ref{app:alternatives}. Training improves excess TV by a median of $0.176$ per target (interval $[0.124, 0.227]$) with the sentence and $0.135$ ($[0.120, 0.173]$) without it, so $77\%$ of the median improvement remains when the list is removed. The median per-target difference between the two improvements is $0.057$ ($[0.010, 0.103]$), so the list contributes, but most of the improvement does not depend on it.

On log-series, a large improvement remains without the sentence, $0.226$ ($[0.204, 0.247]$). On occupancy, the improvement is larger without the sentence (difference $-0.054$, interval excluding zero). Removing the sentence also changes the prompt format, so this test does not separate the content of the list from the change of format.

\subsection{Variety of the Training Targets}
\label{app:variety}

Transfer to unseen families may depend on how varied the training targets are. To test this at matched compute ($600$ steps, same group size and decoding), we compare the five-family retrain with the witness advantage trained on three other training sets, all in the evaluation format. One has a single family (the $16$ geometric targets), one a single target (a geometric target with $p = 0.551$), and one fifteen families ($240$ targets). Table~\ref{tab:variety} gives the results.

\begin{table}[h]
\centering
\caption{Median excess TV on the 100 unseen-family targets after training on sets of different variety, all with prompts in the evaluation format and the same compute. Reduction is relative to the untrained model ($0.589$).}
\label{tab:variety}
\small
\setlength{\tabcolsep}{12pt}
\renewcommand{\arraystretch}{1.1}
\begin{tabular}{@{}lcc@{}}
\toprule
training set & excess TV & reduction \\
\hmidrule
one geometric target & 0.397 & $33\%$ \\
one family (16 geometric targets) & 0.462 & $22\%$ \\
five families (80 targets) & 0.447 & $24\%$ \\
fifteen families (240 targets) & 0.394 & $33\%$ \\
\bottomrule
\end{tabular}
\end{table}

Widening from five to fifteen families lowers excess TV on the typical target (median paired difference $-0.038$, interval $[-0.045, -0.025]$).

The single-target run reaches about the same pooled value as the fifteen-family run, $0.397$ against $0.394$, and a second seed of the single-target run reaches $0.383$. Its gains are narrow, however. Most come from log-series, the one unseen family that shares the geometric family's support and decreasing shape (excess TV from $0.606$ to $0.119$), and from occupancy ($0.615$ to $0.456$). On hypergeometric and maximum-of-dice targets it ends worse than the untrained model ($0.654$ to $0.693$ and $0.662$ to $0.701$), while the five-family run improves every unseen family. The single-target model also does not read the stated parameter. On unseen geometric parameters, the single-family model has excess TV $0.02$ to $0.11$ for $p$ from $0.155$ to $0.914$, while the single-target model is accurate only near its training value and reaches $0.44$ at $p = 0.155$.

Because compute is matched, fewer training targets also mean more visits to each prompt ($2{,}400$ group visits for the single target against $10$ per target for fifteen families), so variety and repetition change together in this comparison.

\section{Opinion Distributions and Structured Outputs}
\label{app:real}

This appendix gives the details of Section~\ref{sec:real}. Table~\ref{tab:opinions} lists the results on both tasks.

\begin{table}[h]
\centering
\caption{Stated opinion distributions and urn draws (top) and structured outputs (bottom), in median excess TV. Trained values are means over seeds, and reductions give the range over seeds: three for the opinion tasks, two for uniform sibling targets, and one for non-uniform sibling targets. The structured-output targets are evaluated on 60 test graphs.}
\label{tab:opinions}
\small
\renewcommand{\arraystretch}{1.1}
\begin{tabular*}{\linewidth}{@{\extracolsep{\fill}}lcccc@{}}
\toprule
opinion evaluation set & Qwen2.5-1.5B & Llama-3.1-8B & trained (ours) & reduction \\
\hmidrule
urn draws & 0.33 & 0.23 & 0.04 & 80 to 91\% \\
GlobalOpinionQA & 0.26 & 0.10 & 0.03 & 86 to 92\% \\
NYTimes task & 0.48 & 0.32 & 0.15 & 61 to 72\% \\
implicit variant & 0.31 & 0.21 & 0.11 & 56 to 69\% \\
\bottomrule
\end{tabular*}

\vspace{9pt}
\begin{tabular*}{\linewidth}{@{\extracolsep{\fill}}lcccc@{}}
\toprule
structured-output target & Qwen2.5-1.5B & best decoding & trained (ours) & reduction \\
\hmidrule
sibling pairs, uniform & 0.568 & 0.540 & 0.358 & 35 to 39\% \\
sibling pairs, non-uniform & 0.533 & 0.458 & 0.280 & 47\% \\
\bottomrule
\end{tabular*}
\end{table}

\paragraph{Opinion distributions.}
The opinion prompts state the answer options, name them as the valid outcomes, and list the target probability of each option to three decimals. The implicit variant removes the probability line. Each opinion prompt is evaluated with $n = 250$ draws instead of $500$. The training set has $1{,}600$ prompts, $800$ from urn draws and $800$ from GlobalOpinionQA, and the test set has $200$ prompts from each task. The urn prompts describe $1{,}000$ urns from the Spectrum Suite, which we split at random into $800$ training and $200$ test urns. The split does not remove duplicates, and $4$ test urns have the same description as a training urn. The GlobalOpinionQA prompts come from $75$ question stems. The test prompts use $18$ of them, and the training prompts use the other $57$, so no test question appears in training.

\paragraph{Comparisons on opinion tasks.}
On the test prompts of both tasks and on the NYTimes task, the trained model has lower excess TV than the best prompting method. On the test prompts, it also has lower excess TV than the best temperature. Across seeds, its paired margin in excess TV is $0.14$ to $0.32$ over the best prompting method and $0.19$ to $0.23$ over the best temperature, and every interval excludes zero. A group-scalar run on the opinion tasks stopped at step $77$ of $600$, when the mean rate of valid output over the previous $20$ steps fell below $0.5$.

\paragraph{Structured outputs.}
\label{app:structured}
We adapt the sibling-discovery task of \citet{nagarajan2025dice}. In their version, the model learns the graph during training, the prompt does not show it, and the model generates two siblings together with their parent. Our prompt instead names a parent and its children and asks for one pair of siblings, so the valid outputs are all pairs of children. We state two kinds of targets over these pairs. A uniform target asks for a pair drawn uniformly at random and does not list the pairs, so the model must build them from the list of children. A non-uniform target lists each pair with its probability. The uniform target tests whether the model finds the valid outputs, and the non-uniform target tests whether it also follows the stated probabilities over them. For uniform targets, we train on $800$ graphs with $8$ to $20$ children and evaluate on $60$ test graphs with $8$ children ($28$ pairs). For non-uniform targets, we train on $200$ graphs with $6$ children and evaluate on $60$ test graphs of the same size.

On uniform targets, training removes $35\%$ and $39\%$ of the excess TV in two seeds, compared with $5\%$ for the best decoding baseline and $8\%$ for the group-scalar reward, which reaches $0.523$. Training also raises the median fraction of valid pairs that the model produces at least once in $500$ draws from $0.71$ to $0.93$ (first seed), so the trained model finds more of the valid outputs. On non-uniform targets, training removes $47\%$ of the excess TV, compared with $14\%$ for the best decoding baseline.